\RequirePackage{amsthm}
\documentclass[sn-basic,Numbered]{sn-jnl}

\usepackage{graphicx}%
\usepackage{multirow}%
\usepackage{amsmath,amssymb,amsfonts}%
\usepackage{mathtools}
\usepackage{mathrsfs}%
\usepackage[title]{appendix}%
\usepackage{xcolor}%
\usepackage{textcomp}%
\usepackage{manyfoot}%
\usepackage{booktabs}%

\usepackage{bm}
\usepackage{caption}
\usepackage{braket}
\usepackage{thm-restate}
\usepackage[linesnumbered, ruled, lined, boxed, noend, algo2e]{algorithm2e}
\usepackage{cleveref}

\newcommand{\R}{\mathbb{R}}

\newcommand{\N}{\mathbb{N}}

\DeclareMathOperator{\dist}{dist}
\DeclareMathOperator{\MST}{MST}

\newcommand{\abs}[1]{\left\lvert #1 \right\rvert}

\renewcommand{\epsilon}{\varepsilon}

\newcommand{\cupdot}{\mathbin{\mathaccent\cdot\cup}}

\renewcommand{\l}{\ell}

\theoremstyle{thmstyleone}%
\newtheorem{theorem}{Theorem}
\newtheorem{claim}[theorem]{Claim}

\theoremstyle{thmstyletwo}%

\theoremstyle{thmstylethree}%

\newenvironment{subproof}[1][\proofname]{%
	\begin{proof}[#1]%
	}{%
	\end{proof}%
}

\begin{document}

\title[Poly-Size ReLU Neural Networks for Maximum Flow Computation]{ReLU Neural Networks of Polynomial Size for Exact Maximum Flow Computation\footnote[2]{This is the authors' accepted manuscript of an article published in the journal \emph{Mathematical Programming}; see \url{https://doi.org/10.1007/s10107-024-02096-x}. An extended abstract of this article appeared in the proceedings of the \emph{24th Conference on Integer Programming and Combinatorial Optimization}~\cite{hertrich2023relu}.}}


\author*[1]{\fnm{Christoph} \sur{Hertrich}}\email{christoph.hertrich@ulb.be}

\author*[2]{\fnm{Leon} \sur{Sering}}\email{sering@math.ethz.ch}

\affil[1]{\orgdiv{D\'epartement de Math\'ematique}, \orgname{Universit\'e libre de Bruxelles}, \orgaddress{
		\country{Belgium}}}

\affil[2]{\orgdiv{Department of Mathematics}, \orgname{ETH Z\"urich}, \orgaddress{
		\country{Switzerland}}}


\abstract{This paper studies the expressive power of artificial neural networks with rectified linear units. In order to study them as a model of \emph{real-valued} computation, we introduce the concept of \emph{Max-Affine Arithmetic Programs} and show equivalence between them and neural networks concerning natural complexity measures. We then use this result to show that two fundamental combinatorial optimization problems can be solved with polynomial-size neural networks. First, we show that for any undirected graph with $n$ nodes, there is a neural network (with fixed weights and biases) of size $\mathcal{O}(n^3)$ that takes the edge weights as input and computes the value of a minimum spanning tree of the graph. Second, we show that for any directed graph with $n$ nodes and $m$ arcs, there is a neural network of size $\mathcal{O}(m^2n^2)$ that takes the arc capacities as input and computes a maximum flow.
Our results imply that these two problems can be solved with strongly polynomial time algorithms that solely use affine transformations and maxima computations, but no comparison-based branchings.}



\maketitle

\section{Introduction}

Artificial neural networks (NNs) achieved breakthrough results in various application domains like computer vision, natural language processing, autonomous driving, and many more~\cite{LeCunBengioHinton:DeepLearning}.
Also in the field of combinatorial optimization (CO), promising approaches to utilize NNs for problem solving or improving classical solution methods have been introduced~\cite{BengioLodiProuvost:MLforCO}.
However, the theoretical understanding of NNs still lags far behind these empirical successes.

All neural networks considered in this paper are \emph{feedforward neural networks with rectified linear unit (ReLU) activations}, one of the most popular models in practice~\cite{glorot2011deep}. These NNs are directed, acyclic, computational graphs in which each edge is equipped with a fixed weight and each node with a fixed bias. Each node (\emph{neuron}) computes an affine transformation of the outputs of its predecessors and applies the ReLU activation function $x\mapsto \max\{0,x\}$ on top. The full NN then computes a function mapping real-valued inputs to real-valued outputs. See \Cref{Sec:Prelim} for a formal definition. A simple example is given in \Cref{Fig:Min2Num}.

\begin{figure}[b]
	\centering
	\begin{minipage}[b]{.34\textwidth}
		\centering
		\includegraphics{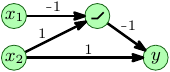}\vspace{0.5em}
		\captionsetup{singlelinecheck=off}
		\caption{A small NN with two input neurons $\bm x_1$ and $\bm x_2$, a single ReLU neuron labelled with the shape of the ReLU function, and one output neuron~$\bm y$. It computes the function%
			\vspace{-0.7em}
			\begin{align*}
				\bm x &\mapsto \bm y\\&= \bm x_2-\max\set{0,\bm x_2-\bm x_1}\\&=-\max\set{-\bm x_2,-\bm x_1} \\&=\min\set{\bm x_1,\bm x_2}.
			\end{align*}\vspace{-2.2em}
		}
		\label{Fig:Min2Num}
	\end{minipage}\hfill
	\begin{minipage}[b]{.6\textwidth}
		\centering
		\includegraphics{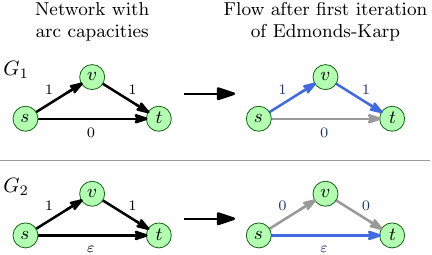}
		\caption{This example shows that the outcome of one iteration of the Edmonds-Karp algorithm for computing a maximum flow depends discontinuously on the arc capacities. Here, a small adjustment of the capacity of arc $st$ leads to a drastic change of the flow after the first iteration.} 
		\label{Fig:edmonds_karp}
	\end{minipage}
\end{figure}

The neurons are commonly organized in \emph{layers}. The \emph{depth}, \emph{width}, and \emph{size} of an NN are defined as the number of layers, the maximum number of neurons per layer, and the total number of neurons, respectively. An important theoretical question about these NNs is concerned with their expressivity: which functions can be represented by an NN of a certain depth, width, or size?

Neural network expressivity has been thoroughly investigated from an approximation point of view. For example, so-called \emph{universal approximation theorems}~\cite{anthony2009neural,cybenko1989approximation,hornik1991approximation} show that every continuous function on a bounded domain can be arbitrarily well approximated with only a single nonlinear layer. However, for a full theoretical understanding of this fundamental machine learning model it is necessary to understand what functions can be \emph{exactly} expressed with different NN architectures. For instance, insights about exact representability have boosted our understanding of the computational complexity of the task to train an NN with respect to both, algorithms~\cite{Arora:DNNwithReLU,khalife2022neural} and hardness results~\cite{GKMR21,froese2021computational,bertschinger2022training,froese2023training}. It is known that a function can be expressed with a ReLU~NN if and only if it is \emph{continuous and piecewise linear} (CPWL)~\cite{Arora:DNNwithReLU}. However, many surprisingly basic questions remain open. For example, it is not known whether two layers of ReLU units (with any width) are sufficient to compute the function \mbox{$f\colon\R^4\to\R$}, \mbox{$x\mapsto\max\{0,x_1,x_2,x_3,x_4\}$~\cite{hertrich2021towards,haase2023lower}}. 

In this paper we explore another fundamental question within the research stream of exact representability: what are families of CPWL functions that can be represented with ReLU~NNs of polynomial size? In other words, using NNs as a model of computation operating on \emph{real} numbers (in contrast to Turing machines or Boolean circuits, which operate on binary encodings), which problems do have polynomial complexity in this model?

Our motivation to study this model stems from a variety of different perspectives, including strongly polynomial time algorithms, arithmetic circuit complexity, parallel computation, and learning theory. 
We believe that classical combinatorial optimization problems are a natural example to study this model of computation because their algorithmic properties are well understood in each of these areas.

If there are polynomial-size NNs to solve a certain problem, and assuming that the weights of these NNs are computable in polynomial time, then there exists a strongly polynomial time algorithm for that problem, simply by executing the NN.\footnote{In circuit complexity language, such a family of neural networks with polynomial-time computable weights would be called a ``uniform'' family. To see that such a family provides a strongly polynomial-time algorithm in the bit model, observe that for any rational input the number of arithmetic operations is bounded by the size of the neural network and hence polynomial. Furthermore, the encoding size of all intermediate numbers in the computation is polynomial in the encoding size of the input because a ReLU network can only perform additions, scalar multiplications, and maxima computations.} However, the converse might be false. This is due to the fact that ReLU~NNs only allow a very limited set of possible operations, namely affine combinations and maxima computations. In particular, every function computed by such NNs is continuous, making it impossible to realize instructions like a simple \textbf{if}-branching based on a comparison of real numbers. In fact, there are related models of computation for which the use of branchings is exponentially powerful~\cite{jerrum1982some}.

For some CO problems, classical algorithms do not involve comparison-based branchings and, thus, can easily be implemented as an NN. This is, for example, true for many dynamic programs. In these cases, the existence of efficient NNs follows immediately. We refer to Hertrich and Skutella~\cite{knapsackPaper} for some examples of this kind. In particular, polynomial-size NNs to compute the length of a shortest path in a network from given arc lengths are possible.

For other problems, like the Minimum Spanning Tree Problem or the Maximum Flow Problem, all classical algorithms use comparison-based branchings.
For example, many maximum flow algorithms use them to decide whether an arc is part of the \emph{residual network}.
More specifically, in the Edmonds-Karp algorithm a slight perturbation (from $0$ to $\epsilon$) in the capacities can lead to different augmenting path and therefore to a completely different intermediate flow; see \Cref{Fig:edmonds_karp}.
Such a discontinuous behavior can never be represented by a ReLU~NN.

\subsection{Our Main Results}

In order to make it easier to think about NNs in an algorithmic way, we introduce the pseudo-code language \emph{Max-Affine Arithmetic Programs} (MAAPs). We show that MAAPs and NNs are equivalent (up to constant factors) concerning three basic complexity measures corresponding to depth, width, and overall size of NNs. Hence, MAAPs serve as a convenient tool for constructing NNs with bounded size and could be useful for further research about NN expressivity beyond the scope of this paper.

We use this result to prove our two main theorems. The first one shows that computing the value of a minimum spanning tree has polynomial complexity on NNs. The proof is based on a result from subtraction-free circuit complexity~\cite{fomin2016subtraction}.

\begin{restatable}{theorem}{thmMST}\label{Thm:MST}
	For a fixed graph with $n$ vertices, there exists an NN of depth $\mathcal{O}(n\log n)$, width $\mathcal{O}(n^2)$, and size $\mathcal{O}(n^3)$ that correctly maps a vector of edge weights to the value of a minimum spanning tree.
\end{restatable}

The second result shows that computing a maximum flow has polynomial complexity on NNs. Since all classical algorithms involve conditional branchings based on the comparison of real numbers, the proof involves the development of a new strongly polynomial maximum flow algorithm which avoids such branchings. While, in terms of standard running times, the algorithm is definitely not competitive with algorithms that exploit comparison-based branchings, it is of independent interest with respect to the structural understanding of flow problems.

\begin{restatable}{theorem}{thmMaxFlow}\label{Thm:MaxFlow}
	Let $G=(V,E)$ be a fixed directed graph with $s,t\in V$, $\abs{V}=n$, and $\abs{E}=m$. There exists an NN of depth and size $\mathcal{O}(m^2n^2)$ and width $\mathcal{O}(1)$ that correctly maps a vector of arc capacities to a vector of flow values in a maximum $s$-$t$-flow.
\end{restatable}

Let us point out that in case of minimum spanning trees, the NN computes only the objective value, while for maximum flows, the NN computes the actual solution. There is a structural reason for this difference: Due to their continuous nature, ReLU~NNs cannot compute a discrete solution vector, like an indicator vector of the optimal spanning tree, because infinitesimal changes of the edge weights would lead to jumps in the output. For the Maximum Flow Problem, however, the optimal flow itself does indeed have a continuous dependence on the arc capacities. This continuity issue could be fixed by allowing neurons with linear threshold activations in the output layer. However, since this significantly changes the considered neural network model, making it way more powerful, we do not consider such networks.

\subsection{Discussion of the Results}

Before presenting our result in more detail, we discuss the significance and limitations of our results from various perspectives.

\paragraph{Learning Theory}

A standard approach to create a machine learning model usually contains the following two steps. The first step is to fix a particular \emph{hypothesis class}. When using NNs, this means to fix an architecture, that is, the underlying graph of the NN. Then, each possible choice of weights and biases of all affine transformations in the network constitutes one hypothesis in the class. The second step is to run an optimization routine to find a hypothesis in the class that fits given training data as accurately as possible.

A core theme in learning theory is to analyze how the choice of the hypothesis class influences different kinds of errors made by the machine learning model. If the chosen hypothesis class is too small, then even the best hypothesis might not be good enough and the model incurs a large \emph{approximation error}. If the chosen hypothesis class is too large, then the model is likely to overfit on the training data resulting in a large \emph{generalization error} when applied to unseen data. Finding a good tradeoff between these different errors is the art of every machine learning practitioner.

Classical learning theory provides a rich toolbox for understanding the effect of a specific hypothesis class on the learning error using notions like PAC learnability, VC-dimension and Rademacher complexity~\cite{Shalev2014:UnderstandingML}. However, these theories struggle to explain the success of modern NN architectures, which are massively overparameterized. Yet they usually do not suffer from overfitting and the generalization error is much lower than expected~\cite{zhang2021understanding}.

While there exist many attempts to explain the mysterious success of modern NNs~\cite{berner2021modern}, there is still a long way ahead of us. Understanding what CPWL functions are actually contained in the hypothesis classes defined by NNs of a certain size (in particular, polynomial size) is a key insight in this direction. We see our combinatorial, exact perspective as a counterbalance and complement to the usual approximate point of view.

\paragraph{Strongly Polynomial Time Algorithms}
As pointed out above, polynomial-size NNs correspond to a subclass of strongly polynomial time algorithms with a very limited set of operations allowed. Given that this subclass stems from one of the most basic machine learning models, our grand vision, to which we contribute with our results, is to understand for different CO problems whether they admit strongly polynomial time algorithms of this type.

Algorithms of this type have not been known before for the two problems considered in this paper. It remains an open question whether such algorithms, and hence, polynomial-size NNs, exist to solve other CO problems for which strongly polynomial time algorithms are known. Can they, for instance, compute the weight of a minimum weight perfect matching in (bipartite) graphs? Can they compute the cost of a minimum cost flow from either node demands or arc costs, while the other of the two quantities is considered to be fixed?

A major open question is also to prove lower bounds on NN sizes. Can we find a family of CPWL functions (corresponding to a CO problem or not) that can be evaluated in strongly polynomial time, but \emph{not} computed by polynomial-size NNs? While proving lower bounds in complexity theory always seems to be a challenging task, we believe that not all hope is lost. For example, in the area of \emph{extended formulations}, it has been shown that there exist problems (in particular, minimum weight perfect matching) which can be solved in strongly polynomial time, but every linear programming formulation to this problem must have exponential size~\cite{rothvoss2017matching}.
Possibly, one can show in the same spirit that also polynomial-size NN representations are not achievable.

\paragraph{Parametric Algorithms} Our results also have an interesting interpretation from the perspective of parametric algorithms. There exists a variety of literature concerning the question how solutions to the Maximum Flow Problem can be represented if the input depends on one or several unknown parameters; see, e.g., the works by Gallo~et~al.~\cite{gallo1989fast} and McCormick~\cite{mccormick1999fast}. Our results imply that there exists such a representation of polynomial size for the most general form of parametric maximum flow problems, namely the one where \emph{all} arc capacities are independent free parameters. This representation is given in the form of a polynomial-size NN and can be evaluated in polynomial time.

\paragraph{Boolean Circuits}
Even though NNs are naturally a model of real computation, it is worth to have a look at their computational power with respect to Boolean inputs. Interestingly, this makes understanding the computational power of NNs much easier. It is easy to see that ReLU~NNs can directly simulate \mbox{AND-,} \mbox{OR-,} and NOT-gates, and thus every Boolean circuit~\cite{mukherjee2017lower}. Hence, in Boolean arithmetics, every problem in P can be solved with polynomial-size~NNs. 

However, requiring the networks to solve a problem for all possible real-valued inputs seems to be much stronger. Consequently, the class of functions representable with polynomial-size NNs is much less understood than in Boolean arithmetics. Our results suggest that rethinking and forbidding basic algorithmic paradigms (like comparison-based branchings) can help towards improving this understanding.

\paragraph{Arithmetic Circuits}

As a circuit model with real-valued computation, ReLU networks are naturally closely related to \emph{arithmetic circuits}.
Just like NNs, arithmetic circuits are computational graphs in which each node computes some arithmetic expression (traditionally addition or multiplication) from the outputs of all its predecessors. Arithmetic circuits are well-studied objects in complexity theory~\cite{shpilka2010arithmetic}. 
Closer to ReLU NNs, there is a special kind of arithmetic circuits called \emph{tropical circuits}~\cite{jukna2015lower}. In contrast to ordinary arithmetic circuits, they only contain maximum (or minimum) gates instead of sum gates and sum gates instead of product gates. Thus, they are arithmetic circuits in the max-plus algebra.

A tropical circuit can be simulated by an NN of roughly the same size since NNs can compute maxima and sums. Thus, NNs are at least as powerful as tropical circuits.
In fact, NNs are strictly more powerful. In particular, lower bounds on the size of tropical circuits do not apply to NNs. A particular example is the computation of the value of a minimum spanning tree. By Jukna and Seiwert~\cite{jukna2019greedy}, no polynomial-size tropical circuit can do this. However, \Cref{Thm:MST} shows that NNs of cubic size (in the number of nodes of the input graph) are sufficient for this task.

The reason for this exponential gap is that, by using negative weights, NNs can realize subtractions (that is, tropical division), which is not possible with tropical circuits; compare the discussion by Jukna and Seiwert~\cite{jukna2019greedy}. However, this is not the only feature that makes NNs more powerful than tropical circuits. In addition, NNs can realize scalar multiplication (tropical exponentiation) with arbitrary real numbers via their weights, which is impossible with tropical circuits. It is unclear to what extent this feature increases the computational power of NNs compared to tropical circuits.

For these reasons, lower bounds from arithmetic circuit complexity do not transfer to NNs. Hence, we identify it as a major challenge to prove meaningful lower bounds of any kind for the computational model of NNs.

\paragraph{Parallel Computation}

Similar to Boolean circuits, NNs are naturally a model of parallel computation by performing all operations within one layer at the same time. Without going into detail here, the depth of an NN is related to the running time of a parallel algorithm, its width is related to the required number of processing units, and its size to the total amount of work conducted by the algorithm. Against this background, a natural goal is to design NNs as shallow as possible in order to make maximal use of parallelization. However, several results in the area of NN expressivity state that decreasing the depth is often only possible at the cost of an \emph{exponential} increase in width; see~\cite{Arora:DNNwithReLU,eldan2016power,liang2017deep,safran2017depth,Telgarsky15,telgarsky2016benefits,yarotsky2017error}.

Interestingly, a related observation can be made for the Maximum Flow Problem using complexity theory. 
By the result of Arora et al.\ \cite{Arora:DNNwithReLU} mentioned above, any CPWL function can be represented with logarithmic depth in the input dimension, while not giving any guarantee on the total size. In particular, this is also true for the function mapping arc capacities to a maximum flow. Hence, NNs with logarithmic depth can solve the Maximum Flow Problem and it arises the question whether such shallow NNs are also possible while maintaining polynomial total size.

The answer is most likely ``no'' since the Maximum Flow Problem is \mbox{\emph{P-complete}}~\cite{goldschlager1982}. P-complete problems are those problems in P that are \emph{inherently sequential}, meaning that there cannot exist a parallel algorithm with polylogarithmic running time using a polynomial number of processors unless the complexity classes P and NC coincide, which is conjectured to be not the case~\cite{greenlaw1995limits}. NNs with polylogarithmic depth and polynomial total size that solve the Maximum Flow Problem, however, would translate to such an algorithm (under mild additional conditions, such as, that the weights of the NN can be computed in polynomial time).
Therefore, we conclude that it is unlikely to obtain NNs for the Maximum Flow Problem that make significant use of parallelization.
In other words, \Cref{Thm:MaxFlow} can probably not be improved to neural networks with polylogarithmic depth while maintaining overall polynomial size.

\subsection{Further Related Work}

Using NNs to solve CO problems started with so-called \emph{Hopfield networks}~\cite{hopfield1985neural} and related architectures in the 1980s and has been extended to general nonlinear programming problems later on~\cite{kennedy1988neural}. Smith~\cite{Smith:NNforCOreview} surveys these early approaches. Also, specific NNs to solve the Maximum Flow Problem have been developed before~\cite{ali1991neural,effati2008neural,nazemi2012}.
Hopfield NNs are special versions of
recurrent neural networks (RNNs) that find solutions to optimization problems by converging towards a minimum
of an energy function. As such, they are conceptually very different from modern feedforward NNs representing a fixed input-output mapping, as we consider them in this paper.

In recent years interactions between NNs and CO have regained a lot of attention in the literature~\cite{BengioLodiProuvost:MLforCO}, for example, for boosting MIP solvers~\cite{lodi2017learning} and solving specific CO problems~\cite{Bello:NeuralCOwithRL,Emami:learningPermutations,Khalil:LearningCOoverGraphs,kool2019attention,nowak:quadrAssignment,Vinyals:PointerNetworks}. These approaches usually are of heuristic nature without quality or running time guarantees.

Concerning the expressivity of ReLU neural networks, various trade-offs between depth and width of NNs~\cite{Arora:DNNwithReLU,eldan2016power,hanin2019universal,hanin2017approximating,liang2017deep,nguyen2018neural,raghu2017expressive,safran2017depth,Telgarsky15,telgarsky2016benefits,yarotsky2017error} and approaches to count and bound the number of linear regions of a ReLU NN~\cite{hanin2019complexity,montufar2014regions,pascanu2014number,raghu2017expressive,serra2018bounding} have been found.
NNs have been studied from a circuit complexity point of view before~\cite{beiu1996circuit,parberry1994circuit,shawe1992classes}. However, these works focus on Boolean circuit complexity of NNs with sigmoid or threshold activation functions. We are not aware of previous work investigating the computational power of ReLU~NNs as arithmetic circuits operating on the real numbers.

For an introduction to classical minimum spanning tree and maximum flow algorithms, we refer to textbooks~\mbox{\cite{AhujaMagnantiOrlin:NetworkFlows,kortevygen,williamson_2019}}. The asymptotically fastest known combinatorial maximum flow algorithm due to Orlin~\cite{orlin2013} runs in $\mathcal{O}(nm)$ time for~$n$ nodes and~$m$ arcs. Recently, almost linear, weakly polynomial algorithms based on interior point methods have been developed~\cite{chen2022maximum}. However, polynomial-size NNs necessarily correspond to strongly polynomial algorithms.

\section{Algorithms and Proof Overview}

In this section we provide an intuitive overview of how we prove our results. The details of the proofs will be presented in \Cref{Sec:MAAP,Sec:CO,Sec:MaxFlow}.

\paragraph{Max-Affine Arithmetic Programs} For the purpose of algorithmic investigations of ReLU~NNs, we introduce the pseudo-code language \emph{Max-Affine Arithmetic Programs} (MAAPs). A MAAP operates on real-valued variables. The only operations allowed in a MAAP are computing maxima and affine transformations of variables as well as parallel and sequential {\bf for} loops with a \emph{fixed}\footnote{In this context, \emph{fixed} means that the number of iterations cannot depend on the specific instance. It can still depend on the size of the instance (e.g., the size of the graph in case of the two CO problems considered in this paper).} number of iterations. In particular, no \textbf{if} branchings are allowed. With a MAAP~$A$, we associate three complexity measures $d(A)$, $w(A)$, and $s(A)$, which can easily be calculated from a MAAP's description. Intuitively, $d(A)$ is related to the \emph{parallel} computation time required to execute the MAAP, $w(A)$ to the number of processors required, and $s(A)$ to the total work performed by the MAAP. However, we calibrate (and name) these measures such that they formally correspond (up to constant factors) to the depth, width, and size of an NN computing the same function as the MAAP does. We formalize this by proving the following proposition, which is similar to the transformation of circuits into \emph{straight-line programs} in Boolean or arithmetic circuit complexity. 

\begin{restatable}{proposition}{NNMAAP}\label{Prop:NN-MAAP}
	For a function $f\colon\R^n\to\R^m$ the following is true.
	\begin{enumerate}[(i)]
		\item If $f$ can be computed by a MAAP $A$, then it can also be computed by an NN with depth at most $d(A)+1$, width at most $w(A)$, and size at most $s(A)$.
		\item If $f$ can be computed by an NN with depth $d+1$, width $w$, and size $s$, then it can also be computed by a MAAP $A$ with $d(A)\leq d$, $w(A)\leq 2w$, and $s(A)\leq 4s$.
	\end{enumerate}
\end{restatable}

The proof of the proposition works by providing explicit constructions to convert a MAAP into an NN (part (i)), and vice versa (part (ii)) while taking care that the different complexity measures translate respectively.

The takeaway from this exercise is that for proving that NNs of a certain size can compute certain functions, it is sufficient to develop an algorithm in the form of a MAAP that computes the same function and to bound its complexity measures $d(A)$, $w(A)$, and $s(A)$.

\paragraph{Minimum Spanning Trees}

A spanning tree in an undirected graph is a set of edges that is connected, spans all vertices, and does not contain any cycle. For given edge weights, the Minimum Spanning Tree Problem is to find a spanning tree with the least possible total edge weight.

Classical algorithms for the Minimum Spanning Tree Problem, for example Kruskal's or Prim's algorithm, use comparison-based branchings to determine the order in which edges are potentially added to the solution. Thus, they cannot be written as a MAAP or implemented as an NN. Instead, \Cref{Thm:MST} can be shown by ``tropicalizing'' a result by Fomin et al.~\cite{fomin2016subtraction} from arithmetic circuit complexity.

To be more precise, Fomin et al.~\cite{fomin2016subtraction} provide a construction of a polynomial-size \emph{subtraction-free arithmetic circuit} (with standard addition, multiplication, and division, but without subtractions) to compute the so-called \emph{spanning tree polynomial} of a graph $(V,E)$. If $\mathcal{T}$ denotes the set of all spanning trees, then this polynomial is
$
\sum_{T\in\mathcal{T}}\ \prod_{e\in T} x_e
$
defined over $\abs{E}$ many variables $x_e$ associated with the edges of the graph.
Tropicalizing this polynomial (to min-plus algebra) results precisely in the tropical polynomial mapping edge weights to the value of a minimum spanning tree:
$
\min_{T\in\mathcal{T}}\ \sum_{e\in T} x_e
$.

In the same way, one can tropicalize the arithmetic circuit provided by Fomin et al.~\cite{fomin2016subtraction}. In fact, every sum gate is just replaced with the small NN from \Cref{Fig:Min2Num} computing the minimum of its inputs, every product with a summation, and every division with a subtraction (realized using negative weights). That way, we obtain a polynomial-size NN to compute the value of a minimum spanning tree from any given edge weights. Note that it is crucial that the circuit is \emph{subtraction-free} because there is no inverse with respect to tropical addition.

While this tropicalization is already sufficient to justify the existence of polynomial-size NNs to compute the value of a minimum spanning tree, to unveil the algorithmic ideas behind this construction, we provide an equivalent, completely combinatorial proof of \Cref{Thm:MST}, making use of MAAPs and \Cref{Prop:NN-MAAP}.

Without loss of generality, we restrict ourselves to complete graphs. Edges missing in the actual input graph can be represented with large weights such that they will never be included in a minimum spanning tree.
For $n=2$ vertices, the MAAP simply returns the weight of the only edge of the graph. For $n\geq 3$, our MAAP is given in \Cref{Alg:MST}.

\begin{algorithm2e}[tb]

	\SetKwFor{ForParallel}{for each}{do parallel}{end}
	\SetKwInput{KwInput}{Input}
	
	\caption{$\MST_n$: Compute the value of a minimum spanning tree for the complete graph on $n\geq3$ vertices.} \label{Alg:MST}
	
	\KwInput{Edge weights $(\bm x_{ij})_{1\leq i < j \leq n}$.}\vspace{0.5em}
	
	$\bm y_n \leftarrow \min_{i\in[n-1]} \bm x_{in}$\label{Line:compy}
	
	\ForParallel{$1\leq i<j\leq n-1$}{
	
		$\bm x'_{ij} \leftarrow \min\set{\bm x_{ij},\ \bm x_{in} + \bm x_{jn} - \bm y_n}$ \label{Line:xprime}
		
	} \vspace{0.5em}
	
	\Return{$\bm y_n + \MST_{n-1}\left((\bm x'_{ij})_{1\leq i < j \leq n-1}\right)$}
	
\end{algorithm2e}

Let us mention that the use of recursions is just a technicality because for each fixed $n$, the recursion can be unrolled and the MAAP can be stated explicitly. In each step, one node of the graph is deleted and all remaining edge weights are updated in such a way that the objective value of the minimum spanning tree problem in the original graph can be calculated from the objective value in the smaller graph. This idea of removing the vertices one by one can be seen as the translation of the so-called \emph{star-mesh transformation} used by Fomin~et~al.~\cite{fomin2016subtraction} into the combinatorial world.

We prove \Cref{Thm:MST} in \Cref{Sec:CO} by, firstly, showing that \Cref{Alg:MST} indeed computes the correct objective value, and secondly, bounding its complexity measures $d(A)$, $w(A)$, and $s(A)$ and applying \Cref{Prop:NN-MAAP}.

\paragraph{Maximum Flows}

For a given directed graph with a source node $s$, a sink node~$t$, and nonnegative capacities on each arc, the Maximum Flow Problem asks to find a flow value for each arc such that no capacity is exceeded, the inflow equals the outflow at each node except for $s$ and~$t$, and the outflow at $s$ (or equivalently the inflow at~$t$) is maximized.

Since classical maximum flow algorithms rely on conditional branchings based on the comparison of real numbers (for instance, to check which arcs are contained in the residual network), we develop a new maximum flow algorithm in the form of a MAAP (see \Cref{Alg:max_flow,Alg:augmenting_flow}), which then translates to an NN of the claimed size by \Cref{Prop:NN-MAAP}. In the description of the algorithm, we assume without loss of generality that for each arc $e=uv\in E$ also its reverse arc~$vu$ is contained in $E$ and let $\vec E$ denote a subset of all arcs containing exactly one arc for each pair of antiparallel arcs. To point out the ability of neural networks to parallelize well, we sometimes use parallel loops even though this does not significantly reduce asymptotic complexity measures in our case.

To explain our algorithm, let us start by recalling the key ideas of the classical Edmonds-Karp-Dinic algorithm~\cite{edmonds1972theoretical,dinic1970algorithm}. The algorithm repeatedly finds a shortest \mbox{$s$-$t$}~path in the residual graph $G^*=(V,E^*)$, and sends the maximum possible amount of flow on such a path, that is, saturates at least one arc. The algorithm terminates by returning a minimum cut once $t$ cannot be reached from $s$ in the residual graph. The key insight in the analysis is that the distance from $s$ to $t$ in the residual graph is non-decreasing, and strictly increases within at most $m$ such iterations. Thus, the number of iterations can be bounded by $\mathcal{O}(nm)$.

A shortest path can be characterized by \emph{distance labels}. The vector $d\in\mathbb{R}^V_+$ is a distance labelling if $d(s)=0$ and $d(v)\le d(u)+1$ for every residual arc $uv\in E^*$. If there exists an $s$-$t$ path $P$ such that $d(v)=d(u)+1$ for every arc in~$P$, then~$P$ is a shortest path. Identifying a shortest path is equivalent to finding distance labels and such a path. We note that the preflow-push algorithm~\cite{goldberg1988new} explicitly relies on using distance labels and pushing flow on residual arcs $uv$ with $d(v)=d(u)+1$.
However, finding such a labelling requires \textbf{if}-branchings as it needs to identify the arcs in $E^*$, that is, arcs with positive residual capacity.

\begin{algorithm2e}[tb]
	\SetKwFunction{FindAugmentingFlow}{FindAugmentingFlow$_k$}
	\SetKwFor{ForParallel}{for each}{do parallel}{end}
	\SetKwInput{KwInput}{Input}
	
	\caption{Compute a maximum flow for a fixed graph~\mbox{$G=(V,E)$}.} \label{Alg:max_flow}
	
	\KwInput{Capacities $(\bm \nu_e)_{e \in E}$.}\vspace{0.5em}
	
	\tcp{Initializing:}
	
	\ForParallel{$uv\in \vec E$}{\label{Line:startIni}
	
		$\bm x_{uv} \leftarrow 0$ 
		\hspace{9.5pt}\tcp{flow; negative value corresponds to flow on $vu$} \label{Line:ini1}
		
		$\bm c_{uv} \leftarrow \bm \nu_{uv}$ 
		\tcp{residual forward capacities}
		
		$\bm c_{vu} \leftarrow \bm \nu_{vu}$ 
		\tcp{residual backward capacities}\label{Line:endIni}
		
	} \vspace{0.5em}           
	
	\tcp{Main part:} 
	
	\For{$k = 1, \dots, n-1$}{\label{Line:outerFor}
	
		\For{$i = 1, \dots, m$}{\label{Line:innerFor}
		
			$(\bm y_e)_{e \in \vec E} \leftarrow \FindAugmentingFlow((\bm c_e)_{e \in E})$ \label{Line:sub}
			
			\tcc{Returns an augmenting flow (respecting the residual capacities) that only uses paths of length exactly $k$ and saturates at least one arc.}
			
			\tcp{Augmenting:}
			
			\ForParallel{$uv \in \vec E$}{\label{Line:startUpdate}
			
				$\bm x_{uv} \leftarrow \bm x_{uv} + \bm y_{uv}$
				\label{Line:beginUpdate}
				
				$\bm c_{uv} \leftarrow \bm c_{uv} - \bm y_{uv}$
				
				$\bm c_{vu} \leftarrow \bm c_{vu} + \bm y_{uv}$\label{Line:endUpdate}
			}
		}
	}
	\Return{$(\bm x_{e})_{e \in \vec E}$}
\end{algorithm2e}

At a high level, our algorithm is similar, but it avoids knowing the arcs in the residual graph and the length $k$ of the shortest residual $s$-$t$ path explicitly. Instead, we guess $k$ in each iteration of the main procedure (\Cref{Alg:max_flow}), making sure that we never overestimate the true length. The guess is initialized as~\mbox{$k=1$} and, in accordance with the Edmonds-Karp-Dinic analysis, we increment~$k$ by one in every $m$ iterations. Based on our guess for $k$, we use a subroutine \texttt{FindAugmentingFlow}$_k$ (\Cref{Alg:augmenting_flow}) with the following feature: if the actual shortest path length is exactly $k$,
the subroutine will send flow from $s$ to $t$ on (possibly multiple) paths of length exactly $k$, saturating at least one arc. If the  shortest path is longer than~$k$, nothing happens in the current iteration.

\begin{algorithm2e}[tp]
	\small
	\SetKwInput{KwInput}{Input}
	\SetKwFor{ForParallel}{for each}{do parallel}{end}
	
	\caption{FindAugmentingFlow$_k$ for a fixed graph $G=(V,E)$ and a fixed length $k$.} \label{Alg:augmenting_flow}
	
	\KwInput{Residual capacities $(\bm c_e)_{e \in E}$.}
	\vspace{0.5em}
	
	\tcp{Initializing:}
	
	\ForParallel{$vw \in \vec E$}{\label{line:start}
	
		$\bm z_{vw} \leftarrow 0$ 
		\quad \tcp{flow in residual network}
		
		$\bm z_{wv} \leftarrow 0$
	}
	\ForParallel{$(i, v) \in [k] \times  (V \setminus \set{t})$}{
	
		$\bm Y_v^i \leftarrow 0$ 
		\hspace{3pt}\tcp{excessive flow at $v$ in iteration $i$ (from $k$ to $1$)}
		
		$\bm a_{i,v}\leftarrow 0$ \label{line:init_fattest}\tcp{initialize fattest path values}
	} 
	\vspace{0.5em}
	
	\tcp{Determining the fattest path values:} 
	
	\ForParallel{$v\in N^-_t$}{ \label{line:begin_fattest_path}
	
		$\bm a_{1, v} \leftarrow \bm c_{vt}$\label{line:fattest_t}
	}
	
	\For{$i = 2, 3, \dots, k$}{\label{line:begin_fattest_for}
	
		\ForParallel{$v \in V \setminus\set{t}$}{
		
			$\bm a_{i, v} \leftarrow \max_{w \in N^+_v \setminus\set{t}}\min\set{\bm a_{i-1,w}, \bm c_{vw}}$ \label{line:end_fattest_path}
		}            
	}
	\vspace{0.5em}
	
	\tcp{Pushing flow of value $\bm a_{k, s}$ from $s$ to $t$:}
	
	$\bm Y_s^k \leftarrow \bm a_{k, s}$ \tcp{excessive flow at $s$} \label{line:begin_pushing}
	
	\For{$i = k, k-1, \dots, 2$}{\label{line:outer_for}
	
		\For{$v \in V\setminus\set{t}$ in index order}{
		
			\For{$w\in N^+_v\setminus\set{t}$ in index order}{
			
				\tcp{Push flow out of $v$ and into $w$:}
				
				$\bm f \leftarrow \min\set{\bm Y_v^i, \bm c_{vw}, \bm a_{i-1, w} - \bm Y_w^{i-1}}$ \tcp{value we can push over $vw$ such that this flow can still arrive at $t$}\label{line:find_push_amount}
				
				$\bm z_{vw} \leftarrow \bm z_{vw} + \bm f$ 
				\label{line:increase_of_z}
				
				$\bm Y_v^i \leftarrow \bm Y_v^i - \bm f$ 
				
				$\bm Y_w^{i-1} \leftarrow \bm Y_w^{i-1} + \bm f\label{line:increase_of_Y}$
			}
		}
	}
	\ForParallel{$v\in N^-_t$}{\label{line:start_t_pushing}
	
		\tcp{Push flow out of $v$ and into $t$:}
		
		$\bm z_{vt} \leftarrow \bm Y_v^1$
		
		$\bm Y_v^1 \leftarrow 0$ 
	} \label{line:end_pushing}
	\vspace{0.5em}
	
	\tcp{Clean-up by bounding:}
	
	\For{$i = 2, 3, \dots, k - 1$}{ \label{line:begin_clean_up}
	
		\For{$w \in V \setminus \set{t}$ in reverse index order}{    
		
			\For{$v \in N^-_w \setminus \set{t}$ in reverse index order}{
			
				$\bm b \leftarrow \min\set{\bm Y_w^i, \bm z_{vw}}$ \tcp{value we can push backwards along $vw$}
				
				$\bm z_{vw} \leftarrow \bm z_{vw} - \bm b$ 
				
				$\bm Y_w^i \leftarrow \bm Y_w^i - \bm b$ 
				
				$\bm Y_v^{i+1} \leftarrow \bm Y_v^{i+1} + \bm b$    \label{line:increase_of_Y_by_b} \label{line:end_clean_up}       
			}
		}
	}
	\vspace{0.5em}
	
	\ForParallel{$uv \in \vec E$}{\label{line:start_y}
	
		$\bm y_{vw} \leftarrow \bm z_{vw} - \bm z_{wv}$ \label{line:end}
	}
	
	\Return{$(\bm y_e)_{e \in \vec E}$}
\end{algorithm2e}

Instead of distance labels, the subroutine computes \emph{fattest path values} $\bm a_{i,v}$ (\cref{line:begin_fattest_path} to \ref{line:end_fattest_path}) that represent the maximum amount of flow that can be sent from $v$ to $t$ on a path of length exactly $i$. Such values can be obtained by a simple dynamic program that is easy to implement as a MAAP. Thus, a path $(s=v_k,v_{k-1},\dots,v_{1},v_0=t)$ of length exactly $k$ is contained in the residual network if and only if $\bm a_{i,v_i} > 0$ for all $i=1,\dots,k$. Our algorithm makes sure that we only send flow along arcs that are contained in such paths. In particular, the current iteration will send positive flow if and only if $\bm a_{k,s}>0$. However, we cannot recover the shortest $s$-$t$ path with capacity $\bm a_{k,s}$. Therefore, in general, flow will not be sent along a single path and the value of the flow output by \texttt{FindAugmentingFlow}$_k$ might be strictly less than $\bm a_{k,s}$.

After computing the $\bm a_{i,v}$ values, \texttt{FindAugmentingFlow}$_k$ greedily pushes flow from $s$ towards $t$, using a lexicographic selection rule to pick the next arc to push flow on (\cref{line:begin_pushing} to \ref{line:end_pushing}). On the high level, this is similar to the preflow-push algorithm, but using the $\bm a_{i,v}$ values that encode the shortest path distance information implicitly. This may leave some nodes with excess flow; 
a final cleanup phase (\cref{line:begin_clean_up} to \ref{line:end_clean_up}) is needed to send the remaining flow back to the source~$s$. 

An example for the \texttt{FindAugmentingFlow}$_k$-subroutine is given in \Cref{fig:find_augmenting_flow}.
We emphasize again that, although the description of the subroutine in the example in \Cref{fig:find_augmenting_flow} seems to rely heavily on the distance of a node to $t$, this information is calculated and used only in an implicit way via the precomputed $\bm a_{i,v}$ values. This way, we are able to implement the subroutine without the usage of comparison-based branchings.

\begin{figure}[tp]
	\centering
	\includegraphics[page=1,width=0.5\linewidth]{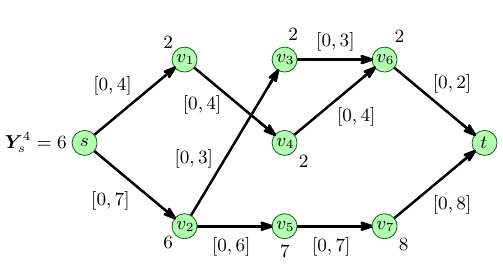}\\[.7em]
	\includegraphics[page=2,width=0.45\linewidth]{find_augmenting_flow}\hfill
	\includegraphics[page=3,width=0.45\linewidth]{find_augmenting_flow}\\[.7em]
	\includegraphics[page=4,width=0.45\linewidth]{find_augmenting_flow}\hfill
	\includegraphics[page=5,width=0.45\linewidth]{find_augmenting_flow}\\[.7em]
	\includegraphics[page=6,width=0.45\linewidth]{find_augmenting_flow}\hfill
	\includegraphics[page=7,width=0.45\linewidth]{find_augmenting_flow}
	\caption{Example of the \texttt{FindAugmentingFlow}$_k$ subroutine for $k=4$. The edge labels in the top figure are the residual capacity bounds in the current iteration. The first step is to compute the fattest path values $\bm a_{i,v}$, which are depicted as node labels in the top figure. The values $\bm Y_v^i$ always denote the excessive flow of a vertex $v$ with distance $i$ from the sink. All values that are not displayed are zero. At $s$, we initialize $\bm Y_s^4=\bm a_{4,s}=6$. Then, excessive flow is pushed greedily towards the sink, as shown in the four figures in the middle. While doing so, we ensure that at each vertex the arriving flow does not exceed its value  $\bm a_{i,v}$. For this reason, flow can get stuck, as it happens at $v_4$ in this example. Therefore, in a final cleanup phase, depicted in the two bottom figures, we push flow back to the source $s$. Observe that the result is an $s$-$t$-flow that is feasible with respect to the residual capacities, uses only paths of length $k=4$, and saturates the arc~$v_6t$.}
	\label{fig:find_augmenting_flow}
\end{figure}

The proof of correctness for our algorithm consists of two main steps. The first step is the analysis of the 
subroutine. This involves carefully showing that the returned flow indeed satisfies flow conservation, is feasible with respect to the residual capacities, uses only arcs that lie on a $s$-$t$-path of length exactly $k$ in the residual network, and most importantly, if such a path exists, it saturates at least one arc.
This last property can be shown using the lexicographic selection rule to pick the next arc to push flow on.
Note that, in general, the subroutine neither returns a single path (as in the Edmonds-Karp algorithm~\cite{edmonds1972theoretical}), nor a blocking flow (as in the Dinic algorithm~\cite{dinic1970algorithm}).
The second main step is to show that, nevertheless, the properties of the subroutine are sufficient to ensure that the distance from $s$ to $t$ in the residual network increases at least every~$m$ iterations, such that we terminate with a maximum flow after $nm$ iterations.

With the correctness of the whole MAAP at hand, \Cref{Thm:MaxFlow} follows by simply counting the complexity measures $d(A)$, $w(A)$, and $s(A)$, and applying \Cref{Prop:NN-MAAP}.

\section{Max-Affine Arithmetic Programs}\label{Sec:MAAP}

One way of specifying an NN is by explicitly writing down the network architecture, weights, and biases, that is, the affine transformations of each layer. However, for NNs that mimic the execution of an algorithm this is very cumbersome and not easy to read. For the purpose of an easier handling of such NNs, we introduce a pseudo-code language, called \emph{Max-Affine Arithmetic Programs (MAAPs)}, and prove that it is essentially equivalent to NNs. To prove this equivalence later in this section, we first provide a formal definition of neural networks.

\subsection{Formal Definition of Neural Networks}
\label{Sec:Prelim}
In this section, we formally define the primary object of study in this paper, using notations similar to \cite[Chapter 20]{Shalev2014:UnderstandingML}. A \emph{feedforward neural network with rectified linear units}, abbreviated by ReLU NN or simply NN, is a directed acyclic graph~$(V,E)$, for which each arc $e \in E$ is equipped with a \emph{weight} $w_e\in\R$ and each
node~$v \in V\setminus V_0$ is equipped with a \emph{bias} $b_v\in\R$. Here, $V_0$ denotes the set of all nodes with no incoming arcs. The nodes~$V$ of an NN are
called \emph{neurons} and the \emph{depth}~$k$ is given by the length of a longest path. The
neurons are partitioned into \emph{layers} $V=V_0\cupdot V_1 \cupdot \cdots \cupdot V_k$ in such a way that the layer
index strictly increases along each arc.\footnote{In other literature arcs are only allowed between successive layers.
	Clearly, this can always be achieved by introducing additional neurons. For our purposes, however, we want to avoid
	this restriction.} In addition, we assume that~$V_k$ is exactly the set of neurons with out-degree zero. The neurons in
$V_0$ and $V_k$ are called \emph{input neurons} and \emph{output neurons}, respectively.
All other neurons in $V\setminus(V_0\cup V_k)$ are so-called \emph{hidden neurons}. Let~$n_\l\coloneqq \abs{V_\l}$
be the number of neurons in layer $\l$. The \emph{width} and \emph{size} of the NN are given by
$\max\set{n_1,\dots,n_{k-1}}$ and~$\sum_{\l=1}^{k-1} n_\l$, respectively.

In our paper it is crucial to distinguish fixed parameters of NNs, like architectural details and weights, from input, activation, and output values of neurons. We denote the latter by bold symbols in order to make the difference visible.

The \emph{forward pass} of an NN is given by a function $\R^{n_0}\to \R^{n_k}$ that is obtained as follows. For an input vector
$\bm x\in\R^{n_0}$ we inductively compute an \emph{activation} $\bm a(v)$ for every~\mbox{$v\in V\setminus V_0$} and an \emph{output} $\bm o(v)$
for every \mbox{$v\in V\setminus V_k$}. First, the output values $\bm o(v)$ of the input neurons are set to the corresponding entries of the input vector $\bm x$. Second, the activation of a neuron~$v\in
V\setminus V_0$ is given by the weighted sum of the outputs of all of its predecessors plus the bias $b_v$. Formally, we set $\bm a(v)\coloneqq b_v +
\sum_{u\colon uv\in E} w_{uv} \bm o(u)$. Third, we apply the so-called \emph{activation function} $\sigma$ to obtain the output of all hidden neurons, i.e., $\bm o(v)\coloneqq\sigma(\bm a(v))$. In this work we only consider the \emph{ReLU function} $\sigma(z)=\max\set{0,z}$ as activation function. Finally, the activation values $\bm a(v)$ of the output neurons $v\in V_k$ provide the \emph{output vector} $\bm y\in \R^{n_k}$. Note that the activation function is not applied to these last activation values.

We point the reader once again to the example in \Cref{Fig:Min2Num}. Note that by our definitions the NN in the figure has depth $2$, width $1$, and size~$1$.

\subsection{Definition of Max-Affine Arithmetic Programs}

MAAPs are a model of real-valued computation, compare the discussion at the end of this section. As such, they perform arithmetic operations on real-valued \emph{variables}. In addition, there are natural- or real-valued \emph{constants} and different kinds of \emph{instructions}. In order to distinguish constants from variables the latter will be denoted by bold symbols. Each MAAP consists of a fixed number of input and output variables, as well as a sequence of (possibly nested) instructions.

In order to describe an algorithm with an arbitrary number of input variables and to be able to measure asymptotic complexity, we specify a \emph{family} of MAAPs that is parametrized by a natural number that determines the number of input variables and is treated like a constant in each single MAAP of the family. A MAAP family then corresponds to a family of NNs. A similar concept is known from circuit complexity, where Boolean circuit families are used to measure complexity; compare~\cite{arora2009computational}.

MAAPs consist of the following types of instructions:

\begin{enumerate}
	\item Assignment: this instruction assigns an expression to an old or new variable. The only two types of allowed expressions are affine combinations or maxima of affine combinations of variables:
	$b + \sum_{j} c_j\bm v_j$ and $\max \left\{\,b^{(i)} + \sum_{j} c_j^{(i)} \bm v_j^{(i)} \mathrel{}\middle|\mathrel{} i = 1, \dots, n\,\right\}$,
	where $n \in \N$, $b, b^{(i)}, c_j, c_j^{(i)} \in \R$ are constants and $\bm v_j, \bm v_j^{(i)} \in \R$ are variables.
	Without loss of generality minima are also allowed.
	\item \textbf{Do-Parallel}: this instruction contains a constant number of blocks of instruction sequences, each separated by an \textbf{and}. These blocks must be executable in parallel, meaning that each variable that is assigned in one block cannot appear in any other block.\footnote{Local variables in different blocks may have the same name if their scope is limited to their block.}
	\item \textbf{For-Do} loop: this is a standard for-loop with a \emph{constant number of iterations} that are executed sequentially.
	\item \textbf{For-Do-Parallel} loop: this is a for-loop with a \emph{constant number of iterations} in which the iterations are executed in parallel. Therefore, variables assigned in one iteration cannot be used in any other iteration.\footnote{Again, local variables within different iterations may have the same name if their scope is limited to their iteration.} 
\end{enumerate}

\Cref{Alg:Example} shows an example MAAP to illustrate the possible instructions.
\begin{algorithm2e}[ht]
	\SetKwFunction{Instructions}{Instructions}
	\SetKwInput{KwInput}{Input}
	\SetKwFor{ForParallel}{for each}{do parallel}{end}
	\SetKwIF{DoParallel}{AndIf}{And}{do parallel}{}{and}{and}{end}

	\caption{Instructions.\label{Alg:Example}}
	
	\KwInput{Input variables $\bm v_1, \bm v_2, \dots, \bm v_n$.}\vspace{0.5em}
	
	\tcp{Assignments and Expressions:}
	
	$\bm x_1 \leftarrow 4 +\sum_{i = 1}^n (-1)^i \cdot \bm v_i$ 
	
	$\bm x_2 \leftarrow \max\set{3 \cdot \bm v_1, -1.5 \cdot \bm v_n, \bm x_1, 5}$
	\vspace{0.5em}
	
	\tcp{For-Do loop:} 
	
	\For{$k = 1, \dots, n-1$}{
	
		$\bm v_{k+1} \leftarrow \bm v_k + \bm v_{k+1}$
	} \vspace{0.5em}
	
	\tcp{Do-Parallel:}       
	
	\uDoParallel{}{
		$\bm y_1 \leftarrow \max \set{\bm x_1, \bm x_2}$}
	\uAnd{
		$\bm y_2 \leftarrow 7$
	} 
	\And{
		$\bm y_3 \leftarrow \sum_{i = 1}^n \bm v_i$
	} \vspace{0.5em}
	
	\tcp{For-Do-Parallel loop:} 
	\ForParallel{$k = 4, \dots, n$}{
	
		$\bm y_{k} \leftarrow \bm v_{k-1} - \bm v_{k}$
		
		$\bm y_k \leftarrow \max\set{\bm y_k, 0}$
		
	} \vspace{0.5em}
	
	\Return{$(\bm y_1, \bm y_2, \dots, \bm y_n)$}
\end{algorithm2e}

Note that we do not allow any \textbf{if}-statements or other branching operations. In other words, the number of executed instructions of an algorithm is always the same independent of the input variables.

\subsection{Equivalence between MAAPs and Neural Networks}

In order to connect MAAPs with NNs, we introduce three complexity measures $d(A)$, $w(A)$, and $s(A)$ for a MAAP~$A$. We will then see that they yield a direct correspondence to depth, width, and size of a corresponding NN.

For these complexity measures for MAAPs, assignments of affine transformations come ``for free'' since in an NN this can be realized ``between layers''. This is a major difference to other (parallel) models of computation, e.g., the parallel random access machine (PRAM)~\cite{greenlaw1995limits}. Apart from that, the complexity measures are recursively defined as follows.

\begin{itemize}
	\item For an assignment $A$ with a maximum or minimum expression of $k\geq 2$ terms we define $d(A) \coloneqq \lceil \log_2 k \rceil$, $w(A)\coloneqq2k$, and $s(A)\coloneqq4k$.
	\item For a sequence $A$ of instruction blocks $B_1, B_2, \dots, B_k$ we define $d(A)\coloneqq\sum_{i=1}^k d(B_i)$, $w(A)\coloneqq\max_{i=1}^k w(B_i)$, and $s(A)\coloneqq\sum_{i=1}^k s(B_i)$.
	\item For a \textbf{Do-Parallel} instruction $A$ consisting of parallel blocks $B_1, B_2, \dots, B_k$ we define $d(A)\coloneqq\max_{i=1}^k d(B_i)$, $w(A)\coloneqq\sum_{i=1}^k w(B_i)$, and $s(A)\coloneqq\sum_{i=1}^k s(B_i)$.
	\item For a \textbf{For-Do} loop $A$ with $k$ iterations that executes block $B_i$ in iteration $i$ we define $d(A)\coloneqq\sum_{i=1}^k d(B_i)$, $w(A)\coloneqq\max_{i=1}^k w(B_i)$, and $s(A)\coloneqq\sum_{i=1}^k s(B_i)$.
	\item For a \textbf{For-Do-Parallel} loop with $k$ iterations that executes block $B_i$ in iteration~$i$ we define $d(A)\coloneqq\max_{i=1}^k d(B_i)$, $w(A)\coloneqq\sum_{i=1}^k w(B_i)$, and $s(A)\coloneqq\sum_{i=1}^k s(B_i)$.
\end{itemize}

With these definitions at hand, we can prove the desired correspondence between the complexities of MAAPs and NNs.

\NNMAAP*

\begin{proof}\
	\begin{enumerate}[(i)]
		\item First note that we can assume without loss of generality that $A$ does not contain \textbf{For-Do} or \textbf{For-Do-Parallel} loops. Indeed, since only a constant number of iterations is allowed in both cases, we can write them as a sequence of blocks or a \textbf{Do-Parallel} instruction, respectively. Note that this also does not alter the complexity measures $d(A)$, $w(A)$, and $s(A)$ by their definition. Hence, suppose for the remainder of the proof that $A$ consists only of assignments and (possibly nested) \textbf{Do-Parallel} instructions.
		
		The statement is proven by an induction on the number of lines of $A$. For the induction base suppose $A$ consists of a single assignment. If this is an affine expression, then an NN without hidden units (and hence with depth 1, width 0, and size 0) can compute $f$. If this is a maximum (or minimum) expression of $k$ terms, then, using the construction of~\cite[Lemma~D.3]{Arora:DNNwithReLU}, an NN with depth $\lceil \log_2 k \rceil + 1$, width $2k$, and size $4k$ can compute~$f$, which settles the induction base.
		
		For the induction step we consider two cases. If $A$ can be written as a sequence of two blocks~$B_1$ and~$B_2$, then, by induction, there are two NNs representing~$B_1$ and~$B_2$ with depth $d(B_i)+1$, width $w(B_i)$, and size $s(B_i)$ for $i=1,2$, respectively. An NN representing $A$ can be obtained by concatenating these two NNs, yielding an NN with depth $d(B_1)+d(B_2)+1=d(A)+1$, width $\max\set{w(B_1), w(B_2)} = w(A)$, and size~$s(B_1)+s(B_2)=s(A)$; cf.~\cite[Lemma~D.1]{Arora:DNNwithReLU}. Otherwise, $A$ consists of a unique outermost \textbf{Do-Parallel} instruction with blocks $B_1, B_2, \dots, B_k$. By induction, there are $k$ NNs representing $B_i$ with depth $d(B_i)+1$, width $w(B_i)$, and size~$s(B_i)$,~$i\in[k]$, respectively. An NN representing $A$ can be obtained by plugging all these NNs in parallel next to each other, resulting in an NN of depth $\max_{i=1}^k d(B_i)+1 = d(A) + 1$, width $\sum_{i=1}^k w(B_i) = w(A)$, and size $\sum_{i=1}^k s(B_i) = s(A)$. This completes the induction.
		
		\item Suppose the NN is given by a directed graph $G=(V,E)$ as described in \Cref{Sec:Prelim}. It is easy to verify that \Cref{Alg:NNtoMAAP} computes $f$ with the claimed complexity measures.
		\qedhere
	\end{enumerate}
\end{proof}

\begin{algorithm2e}[tb]
	\SetKwInput{KwInput}{Input}
	\SetKwFor{ForParallel}{for each}{do parallel}{end}
	
	\caption{A generic MAAP to execute a given NN.\label{Alg:NNtoMAAP}}
	
	\KwInput{Input variables $\bm o(v)$ for $v\in V_0$.}\vspace{0.5em}
	
	\tcp{For each hidden layer:} 
	
	\For{$\ell = 1, \dots, d$}{
	
		\tcp{For each neuron in the layer:}
		
		\ForParallel{$v\in V_\ell$}{
		
			$\bm a(v) \leftarrow b_v + \sum_{u\colon uv\in E} w_{uv} \bm o(u)$
			
			$\bm o(v) \leftarrow \max\set{0, \bm a(v)}$
		}
	} \vspace{0.5em}
	
	\tcp{For each output neuron:} 
	
	\ForParallel{$v\in V_{d+1}$}{
	
		$\bm a(v) \leftarrow b_v + \sum_{u\colon uv\in E} w_{uv} \bm o(u)$
		
	} \vspace{0.5em}
	
	\Return{$(\bm a(v))_{v\in V_{d+1}}$.}
\end{algorithm2e}

\subsection{Relation to Other Models of Real-Valued Computation}

Real-valued computation is often formalized via different versions of the \emph{Real RAM} model~\cite{shamos1979} or the \emph{Blum-Shub-Smale machine}~\cite{blum1989theory}. Defining real models of computation requires a bit of care because it easily happens that they unintentionally become too powerful. As an example, if a model allows both, \emph{indirect indexing}\footnote{This means that the sequence of variable accesses can depend on the input instance, which is common in most programming languages.} and \emph{multiplication} of real variables in constant time, then one can show that every problem solvable in polynomial space can also be solved in polynomial time~\cite{schonhage1979power,pratt1974characterization,hartmanis1974power}; compare also a concise recent discussion in~\cite{erickson2022smoothing}.

For MAAPs, the sequence of variable accesses is indeed predefined and cannot depend on the input. Therefore, our model does not include indirect indexing. Also multiplications of variables is not allowed, as we only allow scalar multiplications with fixed constants. Hence, MAAPs do not suffer from the issue described above. Nevertheless, the precise computational power of MAAPs remains an open question, to which we contribute with two positive results in this article. 

\section{Neural Networks for Minimum Spanning Trees}
\label{Sec:CO}

In this section we prove \Cref{Thm:MST} by showing correctness of \Cref{Alg:MST} and applying \Cref{Prop:NN-MAAP}.

\begin{restatable}{proposition}{MSTMAAP}\label{Prop:MSTMAAP}
	\Cref{Alg:MST} correctly computes the value of a minimum spanning tree in the complete graph on $n$ vertices.
\end{restatable}

\begin{proof}
	We use induction on $n$. The trivial case $n=2$ settles the induction start.
	For the induction step, we separately show that \Cref{Alg:MST} does neither over- nor underestimate the true objective value. For this purpose we show that minimum spanning trees in the original graph and the smaller graph with updated weights can be constructed from each other in a way that is consistent with the computation performed by \Cref{Alg:MST}.
	
	Suppose that the subroutine $\MST_{n-1}$ correctly computes the value of an MST for $n-1$ vertices. We need to show that the returned value $\bm y_n + \MST_{n-1}\left((\bm x'_{ij})_{1\leq i < j \leq n-1}\right)$ is indeed the MST value for $n$ vertices.
	
	First, we show that the value computed by \Cref{Alg:MST} is not larger than the correct objective value. To this end, let $T$ be the set of edges corresponding to an MST of $G$. By potential relabeling of the vertices, assume that $\bm y_n=\bm x_{1n}$. Note that we may assume without loss of generality that $v_1v_n\in T$: if this is not the case, adding it to $T$ creates a cycle in $T$ involving a second neighbor $v_i\neq v_1$ of $v_n$. Removing $v_iv_n$ from $T$ results again in a spanning tree with total weight at most the original weight.
	
	We construct a spanning tree $T'$ of the subgraph spanned by the first $n-1$ vertices as follows: $T'$ contains all edges of $T$ that are not incident with $v_n$. Additionally, for each~$v_iv_n\in T$, except for $v_1v_n$, we add the edge $v_1v_i$ to $T'$. It is immediate to verify that this construction results in fact in a spanning tree.
	We then obtain
	
	\begin{align*}
		\sum_{v_iv_j\in T} \bm x_{ij}
		&= \bm x_{1n} + \sum_{v_iv_n\in T,\ i>1} \bm x_{in} + \sum_{v_iv_j\in T,\ i,j<n} \bm x_{ij}\\
		&= \bm y_n + \sum_{v_iv_n\in T,\ i>1} (\bm x_{in} + \bm x_{1n} - \bm y_n) + \sum_{v_iv_j\in T,\ i,j<n} \bm x_{ij}\\
		&\geq \bm y_n + \sum_{v_iv_n\in T,\ i>1} \bm x'_{1i} + \sum_{v_iv_j\in T,\ i,j<n} \bm x'_{ij}\\
		&= \bm y_n + \sum_{v_iv_j\in T'} \bm x'_{ij}\\
		&\geq \bm y_n + \MST_{n-1}\left((\bm x'_{ij})_{1\leq i < j \leq n-1}\right).
	\end{align*}
	
	Here, the first inequality follows by the way how the values of $\bm x'$ are defined in line~\ref{Line:xprime} and the second inequality follows since $T'$ is a spanning tree of the first $n-1$ vertices and, by induction, the MAAP is correct for up to $n-1$ vertices. This completes the proof that the MAAP does not overestimate the objective value. 
	
	In order to show that the MAAP does not underestimate the true objective value, let~$T'$ be the set of edges of a minimum spanning tree of the first $n-1$ vertices with respect to the updated costs $\bm x'$. Let~\mbox{$E^*\subseteq T'$} be the subset of edges~$v_iv_j$, with \mbox{$1\leq i<j \leq n-1$}, in $T'$ that satisfy \mbox{$\bm x'_{ij}=\bm x_{in} + \bm x_{jn}-\bm y_n$}. Note that, in particular, we have $\bm x'_{ij}=\bm x_{ij}$ for all $v_iv_j\in T'\setminus E^*$, which will become important later. We show that we may assume without loss of generality that $E^*$ only contains edges incident with $v_1$. To do so, suppose there is an edge $v_iv_j\in E^*$ with $2\leq i<j \leq n-1$. Removing that edge from $T'$ disconnects exactly one of the two vertices $v_i$ and $v_j$ from $v_1$; say, it disconnects $v_j$. We then can add $v_1v_j$ to $T'$ and obtain another spanning tree in $G'$. Moreover, by the definition of the weights $\bm x'$ and the choice of~$v_1$, we obtain $\bm x'_{1j}\leq \bm x_{1n} + \bm x_{jn} - \bm y_n \leq \bm x_{in} + \bm x_{jn}-\bm y_n = \bm x'_{ij}$. Hence, the new spanning tree is still minimal. This procedure can be repeated until every edge in $E^*$ is incident with $v_1$.
	
	Now, we construct a spanning tree $T$ in $G$ from $T'$ as follows: $T$ contains all edges of~$T'\setminus E^*$. Additionally, for every $v_1v_i\in E^*$, we add the edge $v_iv_n$ to $T$. Finally, we also add~$v_1v_n$ to $T$. Again it is immediate to verify that this construction results in fact in a spanning tree, and we obtain
	
	\begin{align*}
		\sum_{v_iv_j\in T} \bm x_{ij}
		&= \bm x_{1n} + \sum_{v_1v_i\in E^*} \bm x_{in} + \sum_{v_iv_j\in T'\setminus E^*} \bm x_{ij}\\
		&= \bm y_n + \sum_{v_1v_i\in E^*} (\bm x_{in} + \bm x_{1n} - \bm y_n) + \sum_{v_iv_j\in T'\setminus E^*} \bm x_{ij}\\
		&= \bm y_n + \sum_{v_1v_i\in E^*} \bm x'_{1i} + \sum_{v_iv_j\in T'\setminus E^*} \bm x'_{ij}\\
		&= \bm y_n + \sum_{v_iv_j\in T'} \bm x'_{ij}\\
		&= \bm y_n + \MST_{n-1}\left((\bm x'_{ij})_{1\leq i < j \leq n-1}\right).
	\end{align*}
	
	This shows that the MAAP returns precisely the value of the spanning tree $T$. Hence, its output is at least as large as the value of an MST, completing the second direction.
\end{proof}

Finally, we prove complexity bounds for the MAAP, allowing us to bound the size of the corresponding NN.

\thmMST*
\begin{proof}
	In \Cref{Prop:MSTMAAP}, we have seen that \Cref{Alg:MST} performs the required computation. We show that $d(\MST_n)=\mathcal{O}(n\log n)$, $w(\MST_n)=\mathcal{O}(n^2)$, and $s(\MST_n)=\mathcal{O}(n^3)$. Then, the claim follows by \Cref{Prop:NN-MAAP}.
	
	Concerning the complexity measure $d$, observe that in each recursion the bottleneck is to compute the minimum in line~\ref{Line:compy}. This is of logarithmic order. Since we have $n$ recursions, it follows that $d(\MST_n)=\mathcal{O}(n\log n)$.
	
	Concerning the complexity measure $w$, observe that the bottleneck is to compute the parallel \textbf{for} loop in line~\ref{Line:xprime}. This is of quadratic order, resulting in $w(\MST_n)=\mathcal{O}(n^2)$.
	
	Finally, concerning the complexity measure $s$, the bottleneck is also the parallel \textbf{for} loop in line~\ref{Line:xprime}. Again, this is of quadratic order and since we have $n$ recursions, we arrive at~$s(\MST_n)=\mathcal{O}(n^3)$.
\end{proof}

\section{Neural Networks for Maximum Flows}\label{Sec:MaxFlow}

In this section we show that, given a fixed directed graph with $n$ nodes and $m$ edges, there is a polynomial-size NN computing a function of type~\mbox{$\R^{m}\to\R^{m}$} that maps arc capacities to a corresponding maximum flow. We achieve this by proving correctness of \Cref{Alg:max_flow} and applying \Cref{Prop:NN-MAAP}. Before we start, we formally set definitions and notations around the Maximum Flow Problem.

\subsection{The Maximum Flow Problem}

Let $G = (V, E)$ be a directed graph with a finite node set $V =
\set{v_1, \dots, v_n}$, $n \in \N$, containing a source $s = v_1$, a sink $t = v_n$, and an arc set $E \subseteq
V^2 \setminus \set{vv|v \in V}$ in which each arc~$e \in E$ is equipped with a capacity $\nu_e \geq 0$. We write $m=\abs{E}$ for the number of arcs,~$\delta^+_v$~and~$\delta^-_v$ for the sets of outgoing and incoming arcs of node $v$, as well as $N^+_v$ and~$N^-_v$ for the sets of successor and predecessor nodes of $v$ in $G$, respectively. The distance~$\dist_{G}(v,w)$ denotes the minimum number of arcs on any path from $v$ to $w$ in $G$.

The \emph{Maximum Flow Problem} consists of finding an $s$-$t$-flow $(y_e)_{e \in
	E}$ satisfying $0 \leq y_e \leq \nu_e$ and~\mbox{$\sum_{e \in \delta^-_v} y_e = \sum_{e \in \delta^+_v} y_e$} for all $v \in V
\setminus \set{s,t}$ such that the \emph{flow value} $\sum_{e \in \delta^+_s} y_e - \sum_{e \in \delta^-_s} y_e$ is maximal.

For the sake of an easier notation we assume for each arc $e=uv\in E$ that its reverse arc~$vu$ is also contained in $E$. This is without loss of generality because we can use capacity~$\nu_e=0$ for arcs that are not part of the original set $E$.
In order to avoid redundancy we represent flow only in one arc direction. More precisely, with $\vec E=\set{v_iv_j\in E | i<j}$ being the set of \emph{forward arcs}, we denote a flow by $(y_e)_{e \in \vec E}$. The capacity constraints therefore state that $-\nu_{vu} \leq y_{uv} \leq \nu_{uv}$. Hence, a negative flow value on a forward arc $uv \in \vec E$ denotes a positive flow on the corresponding \emph{backward arc} $vu$.

A crucial construction for maximum flow algorithms is the \emph{residual network}. For a given $s$-$t$-flow $(y_{e})_{e \in \vec E}$, the \emph{residual capacities} are defined as follows. For an arc $uv \in \vec E$ the \emph{residual forward capacity} is given by $c_{uv} \coloneqq \nu_{uv} - y_{uv}$ and the \emph{residual backward capacity} by~$c_{vu} \coloneqq \nu_{vu} + y_{uv}$. The \emph{residual network} consists of all directed arcs with positive residual capacity. Hence, it is given by $G^*=(V, E^*)$ with $E^* \coloneqq \set{e \in E | c_e > 0}$.

\subsection{Analysis of the Subroutine}

We first analyse the subroutine \texttt{FindAugmentingFlow$_k$} given in \Cref{Alg:augmenting_flow}.
The following theorem states that the subroutine indeed computes an augmenting flow fulfilling all required properties.

\begin{samepage}
	\begin{restatable}{theorem}{thmaugmentingflow} \label{Thm:augmenting_flow}
		Let $(\bm c_e)_{e \in E}$ be residual capacities such that the distance between $s$ and $t$ in the residual network $G^*=(V,
		E^*)$ is at least $k$. Then the MAAP given in \Cref{Alg:augmenting_flow} returns an $s$-$t$-flow $\bm y = (\bm y_{uv})_{{uv} \in \vec E}$ with $-\bm c_{vu} \leq \bm y_{uv} \leq \bm c_{uv}$ such that there is positive flow only on arcs that lie on an $s$-$t$-path of length exactly $k$ in $G^*$. If the distance of $s$ and $t$ in the residual network is exactly $k$, then $\bm y$
		has a strictly positive flow value and there exists at least one saturated arc, i.e., one arc $e \in E^*$ with
		$\bm y_{e} = \bm c_{e}$.
	\end{restatable}
\end{samepage}

The proof idea is as follows.
We first show that we correctly track the excessive flow at each vertex $u$ throughout the whole subroutine with variables $\bm Y_u^i$. This will be used to show two essential properties. First, by the way how we bound the amount of our pushes in terms of the fattest flow values, we ensure that the augmenting flow will only be positive on arcs contained in a shortest $s$-$t$-path. Second, by a careful analysis of the cleanup procedure, we obtain that the final flow fulfills flow conservation. It is then easy to see that it also respects the residual capacities.

In order to show that at least one residual arc is saturated we consider a node $v^*$ that has positive excess flow
after the pushing phase. Among these, $v^*$ is chosen as one of the closest nodes to~$t$ in~$G^*$. From all shortest $v^*$-$t$-paths in~$G^*$ we pick the path~$P$
that has lexicographically the smallest string of node indices. As the fattest path from $v^*$ to $t$ has at least the
residual capacity of~$P$ (given by the minimal residual capacity of all arcs along~$P$), the pushing procedure has pushed
at least this value along $P$. Hence, the arc on $P$ with minimal capacity has to be saturated. It is then easy to show
that the clean-up does not reduce the value along $P$.

\begin{proof}[Proof of \Cref{Thm:augmenting_flow}]
	It is easy to check that lines~\ref{line:begin_fattest_path} to \ref{line:end_fattest_path} from the algorithm do
	indeed compute the maximal flow value $\bm a_{i,v}$ that can be send from $v$ to $t$ along a single path (which we call
	the \emph{fattest path}) of length exactly $i$.
	
	In the following we show that in line \ref{line:end} the arc vector $\bm z = (\bm z_e)_{e \in E}$ forms an $s$-$t$-flow in the residual network $G^* = (V,E^*)$ that satisfies $0 \leq \bm z_e \leq \bm c_e$.
	For this, recall that $\dist_{G^*}(s,u)$ denotes the distance from $s$ to $u$ in the residual network.
	
	In order to prove flow conservation of $\bm z$ at all vertices except for $s$ and $t$, we fix some node $u \in V \setminus \set{t}$ and show that
	\begin{equation} \label{eq:Y_is_excess_flow}  
		\bm Y_u^{j} = \begin{cases}
			\sum_{e \in \delta_u^-} \bm z_e - \sum_{e \in \delta_u^+} \bm z_e & \text{if } j = k - \dist_{G^*}(s,u),\\
			0 & \text{otherwise,}
		\end{cases}
	\end{equation}
	holds throughout the execution of the subroutine.
	
	\begin{claim} \label{claim:Y_is_excess}
		\Cref{eq:Y_is_excess_flow} holds after the pushing procedure (lines \ref{line:begin_pushing} to \ref{line:end_pushing}).
	\end{claim}
	
	\begin{subproof}[Proof of \Cref{claim:Y_is_excess}]
		For $j < k-\dist_{G^*}(s,u)$, there does not exist any $u$-$t$-path of length~$j$ (since $j + \dist_{G^*}(s,u) < k \leq \dist_{G^*}(s,t)$). Hence, the fattest
		$u$-$t$-path of length exactly~$j$ has capacity~$\bm a_{j,u}=0$.
		In any iteration that might increase $\bm Y_u^j$, that is, for $i=j+1$, $v \in V \setminus \set{u, t}$, and $w=u$, we have $\bm f = 0$. This implies that $\bm Y_u^j$ remains $0$.
		
		For $j > k - \dist_{G^*}(s,u)$, the intuition is that there is no $s$-$u$-path of length $k - j$ in the residual graph as $k - j < \dist_{G^*}(s,u)$. Therefore, it is not possible that any flow reaches vertex $u$ in $k-j$ iterations and $Y_u^j$ stays $0$. To formally prove $\bm Y_u^j = 0$ in this case, we use a downward induction on $j$ from $k$ to $1$.
		
		In the base case $j=k$, we have that $\bm Y_u^j = \bm Y_u^k = 0$ since for $u \neq s$ it holds that $\bm Y_u^k$ stays $0$ during the whole algorithm.
		
		For the induction step, fix a given $j<k$. Observe that $\bm Y_u^j$ can only be increased in line~\ref{line:increase_of_Y} when we are in iteration $i=j+1$ and $w=u$ is the successor of some vertex $v$ from which we push. The amount $\bm f$, by which we increase $\bm Y_u^j$, can only be positive if the residual capacity $\bm c_{vw}=\bm c_{vu}$ is positive. This implies $\dist_{G^*}(s,u)\leq \dist_{G^*}(s,v)+1$ and thus $i = j+1 > k - \dist_{G^*}(s,u) + 1 \geq k - \dist_{G^*}(s,v)$. We can therefore apply the induction hypothesis and conclude that $\bm Y_v^i=0$. This, however, implies $\bm f=0$ and therefore that $\bm Y_u^j$ remains $0$ too, completing the inductive proof of this case.
		
		Summarizing, we obtain that $\bm Y_u^j$ can only be non-zero for $j = k-\dist_{G^*}(s,u)$. In each iteration
		with $i = k-\dist_{G^*}(s,u) +1$ and $w = u$ we add $\bm f$ to the flow value $\bm z_{vu}$ and the same to $\bm
		Y_u^{k-\dist_{G^*}(s,u)}$ and in each iteration with $i = k-\dist_{G^*}(s,u)$ and $v = u$ we add~$\bm f$ to the flow
		value $\bm z_{uw}$ and subtract $\bm f$ from $\bm Y_u^{k-\dist_{G^*}(s,u)}$. Hence, $\bm Y_u^{k-\dist_{G^*}(s,u)}$
		denotes exactly the excessive flow after the pushing procedure as stated in \eqref{eq:Y_is_excess_flow}.
	\end{subproof}
	
	This claim already shows that $\bm z_e$ can only be positive if $e$ lies on an $s$-$t$-path of length exactly $k$, which is a
	shortest path in the residual network. To see this, let $vw$ be an arc that is not on such a path. In line~\ref{line:find_push_amount}, it either holds that $\bm Y_v^i = 0$ (if $i \neq k - \dist_{G^*}(s,u)$) or $\bm a_{i - 1, w} = 0$ because for $i = k -
	\dist_{G^*}(s,u)$ there is no $w$-$t$-path of length $i-1 = k - \dist_{G^*}(s,u) - 1$ (since otherwise  $vw$ would
	lie on an $s$-$t$-path of length $k$). Thus, $\bm z_{vw}$ will never be increased. As the clean-up only reduces the flow values, $\bm z_{vw}$ will still be $0$ at the end (line \ref{line:end}).
	
	\begin{claim} \label{claim:Y_is_excess_after_clean_up}		\Cref{eq:Y_is_excess_flow} holds in each iteration of the clean-up (lines \ref{line:begin_clean_up} to \ref{line:end_clean_up}).
	\end{claim}
	\begin{subproof}[Proof of \Cref{claim:Y_is_excess_after_clean_up}]
		First, we show that $\bm Y_u^j$ stays $0$ for $j \neq k - \dist_{G^*}(s,u)$ by induction over $j = 2,3, \dots, k$. The
		base case follows immediately as we only subtract $\bm b \geq 0$ from $\bm Y_u^2$. For the induction step we have to
		show that $\bm b = 0$ whenever we add $\bm b$ to $\bm Y_u^j$ in line \ref{line:increase_of_Y_by_b}. In all iterations
		with $i = j-1$ and $v = u$ we either have $\bm z_{uw} = 0$ or $\bm Y_w^i = 0$. The reason for this is that $\bm z_{uw}
		> 0$ implies that $uw$ lies on a shortest $s$-$t$-path, which means that $\dist_{G^*}(s,w) = \dist_{G^*}(s,u) + 1$, and
		hence, $i = j-1 \neq k - \dist_{G^*}(s,u) - 1 = k - \dist_{G^*}(s,w)$. By induction this means that $\bm Y_w^i = 0$.
		Either way this implies $\bm b = 0$.
		
		\Cref{eq:Y_is_excess_flow} holds for $j = k - \dist_{G^*}(s,u)$ since for $e \in \delta_u^-$ the value $\bm b$ is only
		possibly positive for $i = k - \dist_{G^*}(s,u)$ and then it is subtracted from $\bm z_e$ as well as from $\bm Y_u^{k -\dist_{G^*}(s,u)}$. For $e
		\in \delta_u^+$, the value $\bm b$ can only be positive for $i = k - \dist_{G^*}(s,u) + 1$, and hence, $\bm b$ is subtracted from $\bm z_e$
		exactly when it is added to $\bm Y_u^{k - \dist_{G^*}(s,u)}$.
	\end{subproof}
	
	Next, we show that at the end of the subroutine it holds that $\bm Y_u^j = 0$ for all $j$, in particular also for $j = k -\dist_{G^*}(s,u)$. The only exception of this is
	$\bm Y_s^k$. To see this, first observe that during the clean-up, $\bm Y_u^{k - \dist_{G^*}(s,u)}$ is maximal after iteration $i = k
	- \dist_{G^*}(s,u)-1$ and does not increase anymore for $i \geq k - \dist_{G^*}(s,u)$. At the start of iteration $i = k - \dist_{G^*}(s,u)$ it holds due to \eqref{eq:Y_is_excess_flow} that
	\[\sum_{e \in \delta_u^-} \bm z_e \geq \bm Y_u^{k - \dist_{G^*}(s,u)}.\]
	Hence, for $i = k - \dist_{G^*}(s,u)$ and $w = u$ there is one iteration for all $e \in \delta_u^-$ and within this
	iteration $\bm Y_u^{k - \dist_{G^*}(s,u)}$ is reduced by $\bm z_e$ until $\bm Y_u^{k - \dist_{G^*}(s,u)} = 0$. This shows
	that after all iterations with $i = k - \dist_{G^*}(s,u)$ it holds that $\bm Y_u^{k - \dist_{G^*}(s,u)} = 0$. Together with~\eqref{eq:Y_is_excess_flow}, this immediately
	implies flow conservation of $(\bm z_e)_{e \in E}$.
	
	Finally, in order to show that $0 \leq \bm z_e \leq \bm c_e$, note that $\bm z_e$ is initialized with $0$ and it is only
	increased in line~\ref{line:increase_of_z} of the unique iteration with $vw = e$ and $i = k-\dist_{G^*}(s,v)$, as we have argued in
	the proof of \Cref{claim:Y_is_excess}. In this iteration we have that $\bm f \leq \bm c_e$, which immediately shows that
	$0 \leq \bm z_e \leq \bm c_e$.
	
	It only remains to show that at least one residual arc is saturated. To this end, suppose that the distance
	of $s$ and $t$ in $G^*$ is $k$, which means that there exists at least one $s$-$t$-path of length exactly $k$ with a strictly
	positive residual capacity on all arc along this path.
	
	Let us consider the set $\set{(v, i) |\bm Y^i_v  > 0 \text{ after the pushing procedure}}$. These are all nodes that need to be
	cleaned up in order to restore flow conservation, paired with their distance to $t$. Let $(v^*, i^*)$ be a tuple of this set such that $i^*$ is minimal.
	In other words, $v^*$ is a node that is closest to $t$ among theses nodes.
	We manually set $(v^*, i^*)$ to $(s, k)$ in the case that the set is empty.
	
	\begin{claim} \label{claim:sat_one} Some arc on a shortest path from $v^*$ to $t$ in $G^*$ is saturated by $\bm y_e$.
	\end{claim}
	\begin{subproof}[Proof of \Cref{claim:sat_one}]
		Among all these paths between $v^*$ and $t$ of length $i^*$ we consider the path $P$ which has lexicographically the
		smallest string of node indices. Let $\bm c_{\min}$ be the minimal residual capacity along this path $P$.
		
		For all nodes $v$ along $P$ (including $v^*$) we have $\bm a_{i,v} \geq \bm c_{\min}$, where $i$ is the distance from~$v$ to $t$ along $P$, since the fattest path from $v$ to $t$ has to have at least the residual capacity of $P$.
		
		After the pushing procedure it holds that $\bm z_e \geq \bm c_{\min}$ for all $e \in P$. This is true for the first arc
		on $P$, since we have excess flow at node $v^*$ remaining (after pushing), hence, we certainly pushed at least $\bm
		c_{\min} \leq \bm a_{i^*-1,w}$ into the first arc $v^*w$ of $P$. (This is also true if $v^* = s$.) For the remaining
		arcs of $P$ it is true, because by the lexicographical minimality of $P$, the algorithm always pushes a flow value that is greater or equal to $\bm c_{\min}$ first along the next arc on $P$.
		
		During the clean-up, this property remains valid as we only reduce flow on arcs that have a distance of more than $i^*$ from $t$.
		
		Hence, an arc $e \in P$ with $\bm c_e = \bm c_{\min}$ is saturated at the very end of the subroutine.
	\end{subproof}
	In conclusion, $\bm y$ is a feasible $s$-$t$-flow in the residual network that has positive value only on paths of length $k$ and saturates at least one arc. This finalizes the proof of \Cref{Thm:augmenting_flow}.
\end{proof}

\subsection{Analysis of the Main Routine}

The following theorem states that \Cref{Alg:max_flow} correctly computes a maximum flow. It turns out that the properties proven in \Cref{Thm:augmenting_flow} about the subroutine \texttt{FindAugmentingFlow$_k$} are in fact sufficient to obtain correctness of the main routine as for the algorithms by Edmonds-Karp and Dinic; see, e.g., \cite{kortevygen}.

\begin{restatable}{theorem}{thmmaapcorrectness}\label{Thm:MAAPCorrectness}
	Let $G=(V,E)$ be a fixed directed graph with $s,t\in V$. For capacities $(\bm\nu_e)_{e \in E}$ as input, the MAAP given by \Cref{Alg:max_flow} returns a maximum $s$-$t$-flow~$(\bm x_{e})_{e \in \vec E}$.
\end{restatable}

\begin{proof}
	It is a well-known fact that a feasible $s$-$t$-flow is maximum if and only if the corresponding residual network does not contain any $s$-$t$-path, see e.g.~\cite[Theorem~8.5]{kortevygen}. Since any simple path has length at most $n-1$, it suffices to show the following claim.
	\begin{claim}\label{Claim:outer}
		After iteration~$k$ of the \textbf{for} loop in line~\ref{Line:outerFor} of \Cref{Alg:max_flow}, $\bm x$ is a feasible $s$-$t$-flow with corresponding residual capacities $\bm c$ such that no $s$-$t$-path of length at most $k$ remains in the residual network.
	\end{claim}
	
	Given a residual network $(V,E^*)$, let $E^*_k$ be the set of arcs that lie on an $s$-$t$-path of length exactly $k$. If the distance from $s$ to $t$ is exactly $k$, then these arcs coincide with the arcs of the so-called \emph{level graph} used in Dinic's algorithm, compare \cite{dinic1970algorithm, kortevygen}.
	
	We will show \Cref{Claim:outer} about the outer \textbf{for} loop by induction on $k$ using a similar claim about the inner \textbf{for} loop.	
	
	\begin{claim}\label{Claim:inner}
		Suppose, at the beginning of an iteration of the \textbf{for} loop in line~\ref{Line:innerFor}, it holds that
		\begin{enumerate}[(i)]
			\item $\bm x$ is a feasible $s$-$t$-flow with corresponding residual capacities $\bm c$, and\label{firstitem}
			\item the length of the shortest $s$-$t$-path in the residual network is at least $k$.\label{lastitem}
		\end{enumerate}
		Then, after that iteration, properties \eqref{firstitem} and \eqref{lastitem} do still hold. Moreover, if $E^*_k$ is nonempty, then its cardinality is strictly reduced by that iteration.
	\end{claim}
	
	\begin{subproof}[Proof of \Cref{Claim:inner}]
		Since \eqref{firstitem} and $\eqref{lastitem}$ hold at the beginning of the iteration, \Cref{Thm:augmenting_flow} implies that the flow $\bm y$ found in line~\ref{Line:sub} fulfills flow conservation and is bounded by $-\bm c_{vu} \leq \bm y_{uv} \leq \bm c_{uv}$ for each $uv\in \vec{E}$. Hence, we obtain that, after updating $\bm x$ and $\bm c$ in lines \ref{Line:startUpdate} to \ref{Line:endUpdate}, $\bm x$ is still a feasible flow that respects flow conservation and capacities, and $\bm c$ are the corresponding new residual capacities. Thus, property \eqref{firstitem} is also true at the end of the iteration.
		
		Let $G^*=(V,E^*)$ and $\tilde{G}^*=(V,\tilde{E}^*)$ be the residual graphs before and after the iteration, respectively. Let $E^*_k$ and $\tilde E^*_k$ be the set of arcs on $s$-$t$-paths of length $k$ in $G^*$ and $\tilde{G}^*$, respectively. Finally, let $E'$ be the union of $E^*$ with the reverse arcs of $E^*_k$ and let $G'=(V,E')$.
		
		Since, by \Cref{Thm:augmenting_flow}, we only augment along arcs in $E^*_k$, it follows that $\tilde{E}^*\subseteq E'$. Let $P$ be a shortest $s$-$t$-path in $G'$ and suppose for contradiction that $P$ contains an arc that is not in $E^*$. Let $e=uv$ be the first of all such arcs of $P$ and let $P_u$ be the subpath of $P$ until node $u$. Then the reverse arc $vu$ must be in $E^*_k$. In particular, $\dist_{G^*}(s,v)<\dist_{G^*}(s,u)\leq \abs{E(P_u)}$, where the second inequality follows because $P_u$ uses only arcs in $E^*$. Hence, replacing the part of $P$ from $s$ to $v$ by a shortest $s$-$v$-path in $G^*$ reduces the length of $P$ by at least two, contradicting that $P$ is a shortest path in~$G'$.
		
		Thus, all shortest paths in $G'$ only contain arcs from $E^*$. In particular, they have length at least $k$. Hence, all paths in $G'$ that contain an arc that is not in $E^*$ have length larger than $k$. Since $\tilde{E}^*\subseteq E'$, this also holds for paths in $\tilde{G}^*$, which implies \eqref{lastitem}. It also implies that $\tilde E^*_k\subseteq E^*_k$. Moreover, by \Cref{Thm:augmenting_flow}, if $E^*_k$ is nonempty, at least one arc of $E^*_k$ is saturated during the iteration, and thus removed from $E^*_k$. Thus, the cardinality of $E^*_k$ becomes strictly smaller.
	\end{subproof}
	
	Using \Cref{Claim:inner}, we are now able to show \Cref{Claim:outer}.
	
	\begin{subproof}[Proof of \Cref{Claim:outer}]
		We use induction on~$k$. For the induction start, note that before entering the \textbf{for} loop in line~\ref{Line:outerFor}, that is, so to speak, after iteration $0$, obviously no $s$-$t$-path of length~$0$ can exist in the residual network. Also note that after the initialization in lines~\ref{Line:startIni} to~\ref{Line:endIni}, $\bm x$ is the zero flow, which is obviously feasible, and $\bm c$ contains the corresponding residual capacities.
		
		For the induction step, consider the $k$-th iteration. By the induction hypothesis, we know that, at the beginning of the $k$-th iteration, $\bm x$ is a feasible $s$-$t$-flow with corresponding residual capacities $\bm c$ and the distance from $s$ to $t$ in the residual network is at least $k$. Observe that by \Cref{Claim:inner}, these properties are maintained throughout the entire $k$-th iteration. In addition, observe that at the beginning of the $k$-th iteration, we have $\abs{E^*_k}\leq m$. Since, due to \Cref{Claim:inner}, $\abs{E^*_k}$ strictly decreases with each inner iteration until it is zero, it follows that after the $m$ inner iterations, the residual network does not contain an $s$-$t$-path of length $k$ any more, which completes the induction.
	\end{subproof}
	
	Since any simple path has length at most~$n-1$, \Cref{Claim:outer} implies that, at the end of iteration~$k=n-1$, the nodes $s$ and $t$ must be disconnected in the residual network. Hence, \Cref{Alg:max_flow} returns a maximum flow, which concludes the proof of \Cref{Thm:MAAPCorrectness}.
\end{proof}

By applying the definition of our complexity measures to \Cref{Alg:max_flow} and the subroutine, we prove the following bounds.

\begin{restatable}{theorem}{thmcomplexity}\label{Thm:MAAPComplexity}
	For the complexity measures of the MAAP $A$ defined by \Cref{Alg:max_flow} it holds that $d(A), s(A) \in\mathcal{O}(n^2m^2)$ and~\mbox{$w(A) \in \mathcal{O}(n^2)$}.
\end{restatable}

\begin{proof}
	We first analyze the MAAP~$A'$ given in \Cref{Alg:augmenting_flow}.	
	Concerning the complexity measures $d$ and $s$, the bottleneck of \Cref{Alg:augmenting_flow} is given by the two blocks consisting of lines~\ref{line:outer_for} to~\ref{line:increase_of_Y} as well as lines~\ref{line:begin_clean_up} to~\ref{line:increase_of_Y_by_b}. Each of these blocks has $\mathcal{O}(km)$ sequential iterations and the body of the innermost \textbf{for} loop has constant complexity. Thus, we have $d(A'), s(A')\in\mathcal{O}(km)\subseteq\mathcal{O}(nm)$ for the overall subroutine.
	
	Concerning measure $w$, the bottleneck is in fact the initialization in lines~\ref{line:start} to~\ref{line:init_fattest}, such that we have $w(A')\in\mathcal{O}(m+kn)\subseteq\mathcal{O}(n^2)$.
	
	Now consider the main routine in \Cref{Alg:max_flow}. For all three complexity measures, the bottleneck is the call of the subroutine in line~\ref{Line:sub} within the two \textbf{for} loops. Since we have a total of $\mathcal{O}(nm)$ sequential iterations, the claimed complexity measures follow. Note that the parallel \textbf{for} loops in lines~\ref{Line:startIni} and~\ref{Line:startUpdate} do not increase the measure $w(A)\in\mathcal{O}(n^2)$.
\end{proof}

We remark that the reported complexity $w(A)\in\mathcal{O}(n^2)$ in \Cref{Thm:MAAPComplexity} is actually suboptimal and can be replaced with $\mathcal{O}(1)$ instead, if the parallel \textbf{for} loops in the MAAP and the subroutine are replaced with sequential ones. The asymptotics of $d(A)$ and $s(A)$ remain unchanged because the bottleneck parts of the MAAP are already of sequential nature. Still, we used parallel \textbf{for} loops in \Cref{Alg:max_flow} and the subroutine in order to point out at which points the ability of NNs to parallelize can be used in a straightforward way.
Combining the previous observations with \Cref{Prop:NN-MAAP} we obtain \Cref{Thm:MaxFlow}.

\thmMaxFlow*

Note that the objective value can be computed from the solution by simply adding up all flow values of arcs leaving the source node $s$. Therefore, this can be done by an NN with the same asymptotic size bounds, too.

Finally, observe that the total computational work of \Cref{Alg:max_flow}, represented by $s(A)$ and, equivalently, the size of the resulting NN, differs only by a factor of $n$ from the standard running time bound $\mathcal{O}(nm^2)$ of the Edmonds-Karp algorithm; see~\cite[Corollary~8.15]{kortevygen}. While the number of augmenting steps is in $\mathcal{O}(nm)$ for both algorithms, the difference lies in finding the augmenting flow. While the Edmonds-Karp algorithm finds the shortest path in the residual network in $\mathcal{O}(m)$ time, our subroutine requires $\mathcal{O}(mn)$ computational work.

\section{Future Research}

Our grand vision, to which we contribute in this paper, is to determine how the computational model defined by ReLU~NNs compares to strongly polynomial time algorithms. In other words, is there a CPWL function (related to a CO problem or not) which can be evaluated in strongly polynomial time, but for which no polyonmial-size NNs exist?
Resolving this question involves, of course, the probably challenging task to prove nontrivial lower bounds on the size of NNs to compute certain functions. Particular candidate problems, for which the existence of polynomial-size NNs is open, are, for example, the Assignment Problem, different versions of Weighted Matching Problems, or Minimum Cost Flow Problems.

Another direct question for further research is to what extent the size of the NN constructions in this paper can be improved. For example, is the size of $\mathcal{O}(n^2m^2)$ to compute maximum flows best possible or are smaller constructions conceivable? Even though highly parallel architectures (polylogarithmic depth) with polynomial width are unlikely,
is it still possible to make use of NNs' ability to parallelize and find a construction with a depth that is a polynomial of lower degree than~$n^2m^2$?

We hope that this paper promotes further research about this intriguing model of computation defined through neural networks and its connections to classical (combinatorial optimization) algorithms.

\backmatter

%
%
%

\bmhead{Disclosure of Funding and Competing Interests}
A large portion of this work was completed while both authors were affiliated with TU Berlin and while Christoph Hertrich was affiliated with the London School of Economics and Political Science.
Christoph Hertrich acknowledges funding by DFG-GRK 2434 Facets of Complexity and by the European Research Council (ERC) under the European Union's Horizon 2020 research and innovation programme (grant agreement ScaleOpt-757481). Leon Sering acknowledges funding by DFG Excellence Cluster MATH+ (EXC-2046/1, project ID: 390685689). The authors declare that they do not have further competing interests.

\bmhead{Acknowledgments}

We thank Max Klimm, Jennifer Manke, Arturo Merino, Martin Skutella, and L{\'a}szl{\'o} V{\'e}gh for many inspiring and fruitful discussions and valuable comments. We thank the anonymous reviewers of both, the conference and the journal version, for their helpful comments to improve the article.

%
%
%
%
%
%


\bibliography{neuralmaxflow}

\begin{thebibliography}{75}
\providecommand{\natexlab}[1]{#1}
\providecommand{\url}[1]{{#1}}
\providecommand{\urlprefix}{URL }
\providecommand{\doi}[1]{\url{https://doi.org/#1}}
\providecommand{\eprint}[2][]{\url{#2}}
 \bibcommenthead

\bibitem[{Ahuja et~al(1993)Ahuja, Magnanti, and
  Orlin}]{AhujaMagnantiOrlin:NetworkFlows}
Ahuja RK, Magnanti TL, Orlin JB (1993) Network flows: theory, algorithms, and
  applications. Prentice Hall, Upper Saddle River, New Jersey, USA

\bibitem[{Ali and Kamoun(1991)}]{ali1991neural}
Ali MM, Kamoun F (1991) A neural network approach to the maximum flow problem.
  In: IEEE Global Telecommunications Conference GLOBECOM'91: Countdown to the
  New Millennium. Conference Record, pp 130--134

\bibitem[{Anthony and Bartlett(1999)}]{anthony2009neural}
Anthony M, Bartlett PL (1999) Neural network learning: Theoretical foundations.
  Cambridge University Press

\bibitem[{Arora et~al(2018)Arora, Basu, Mianjy, and
  Mukherjee}]{Arora:DNNwithReLU}
Arora R, Basu A, Mianjy P, et~al (2018) Understanding deep neural networks with
  rectified linear units. In: International Conference on Learning
  Representations

\bibitem[{Arora and Barak(2009)}]{arora2009computational}
Arora S, Barak B (2009) Computational complexity: a modern approach. Cambridge
  University Press

\bibitem[{Beiu and Taylor(1996)}]{beiu1996circuit}
Beiu V, Taylor JG (1996) On the circuit complexity of sigmoid feedforward
  neural networks. Neural Networks 9(7):1155--1171

\bibitem[{{Bello} et~al(2016){Bello}, {Pham}, {Le}, {Norouzi}, and
  {Bengio}}]{Bello:NeuralCOwithRL}
{Bello} I, {Pham} H, {Le} QV, et~al (2016) Neural combinatorial optimization
  with reinforcement learning. arXiv:161109940

\bibitem[{{Bengio} et~al(2018){Bengio}, {Lodi}, and
  {Prouvost}}]{BengioLodiProuvost:MLforCO}
{Bengio} Y, {Lodi} A, {Prouvost} A (2018) Machine learning for combinatorial
  optimization: a methodological tour d'horizon. arXiv:181106128

\bibitem[{Berner et~al(2021)Berner, Grohs, Kutyniok, and
  Petersen}]{berner2021modern}
Berner J, Grohs P, Kutyniok G, et~al (2021) The modern mathematics of deep
  learning. arXiv:210504026

\bibitem[{Bertschinger et~al(2022)Bertschinger, Hertrich, Jungeblut, Miltzow,
  and Weber}]{bertschinger2022training}
Bertschinger D, Hertrich C, Jungeblut P, et~al (2022) Training fully connected
  neural networks is $\exists\mathbb{R}$-complete. arXiv:220401368

\bibitem[{Blum et~al(1989)Blum, Shub, and Smale}]{blum1989theory}
Blum L, Shub M, Smale S (1989) On a theory of computation and complexity over
  the real numbers: {NP}-completeness, recursive functions and universal
  machines. Bulletin of the American Mathematical Society 21(1):1--46

\bibitem[{Chen et~al(2022)Chen, Kyng, Liu, Peng, Gutenberg, and
  Sachdeva}]{chen2022maximum}
Chen L, Kyng R, Liu YP, et~al (2022) Maximum flow and minimum-cost flow in
  almost-linear time. arXiv:220300671

\bibitem[{Cybenko(1989)}]{cybenko1989approximation}
Cybenko G (1989) Approximation by superpositions of a sigmoidal function.
  Mathematics of control, signals and systems 2(4):303--314

\bibitem[{Dinic(1970)}]{dinic1970algorithm}
Dinic EA (1970) Algorithm for solution of a problem of maximum flow in a
  network with power estimation. Soviet Mathematics Doklady 11:1277--1280

\bibitem[{Edmonds and Karp(1972)}]{edmonds1972theoretical}
Edmonds J, Karp RM (1972) Theoretical improvements in algorithmic efficiency
  for network flow problems. Journal of the ACM 19(2):248--264

\bibitem[{Effati and Ranjbar(2008)}]{effati2008neural}
Effati S, Ranjbar M (2008) Neural network models for solving the maximum flow
  problem. Applications and Applied Mathematics 3(3):149--162

\bibitem[{Eldan and Shamir(2016)}]{eldan2016power}
Eldan R, Shamir O (2016) The power of depth for feedforward neural networks.
  In: Conference on Learning Theory, pp 907--940

\bibitem[{{Emami} and {Ranka}(2018)}]{Emami:learningPermutations}
{Emami} P, {Ranka} S (2018) Learning permutations with sinkhorn policy
  gradient. arXiv:180507010

\bibitem[{Erickson et~al(2022)Erickson, Van Der~Hoog, and
  Miltzow}]{erickson2022smoothing}
Erickson J, Van Der~Hoog I, Miltzow T (2022) Smoothing the gap between {NP} and
  {ER}. SIAM Journal on Computing

\bibitem[{Fomin et~al(2016)Fomin, Grigoriev, and
  Koshevoy}]{fomin2016subtraction}
Fomin S, Grigoriev D, Koshevoy G (2016) Subtraction-free complexity, cluster
  transformations, and spanning trees. Foundations of Computational Mathematics
  16(1):1--31

\bibitem[{Froese and Hertrich(2023)}]{froese2023training}
Froese V, Hertrich C (2023) Training neural networks is np-hard in fixed
  dimension. arXiv:230317045

\bibitem[{Froese et~al(2021)Froese, Hertrich, and
  Niedermeier}]{froese2021computational}
Froese V, Hertrich C, Niedermeier R (2021) The computational complexity of relu
  network training parameterized by data dimensionality. arXiv:210508675

\bibitem[{Gallo et~al(1989)Gallo, Grigoriadis, and Tarjan}]{gallo1989fast}
Gallo G, Grigoriadis MD, Tarjan RE (1989) A fast parametric maximum flow
  algorithm and applications. SIAM Journal on Computing 18(1):30--55

\bibitem[{Glorot et~al(2011)Glorot, Bordes, and Bengio}]{glorot2011deep}
Glorot X, Bordes A, Bengio Y (2011) Deep sparse rectifier neural networks. In:
  14th International conference on artificial intelligence and statistics, pp
  315--323

\bibitem[{Goel et~al(2021)Goel, Klivans, Manurangsi, and Reichman}]{GKMR21}
Goel S, Klivans AR, Manurangsi P, et~al (2021) Tight hardness results for
  training depth-2 {ReLU} networks. In: 12th Innovations in Theoretical
  Computer Science Conference (ITCS~'21)

\bibitem[{Goldberg and Tarjan(1988)}]{goldberg1988new}
Goldberg AV, Tarjan RE (1988) A new approach to the maximum-flow problem.
  Journal of the ACM (JACM) 35(4):921--940

\bibitem[{Goldschlager et~al(1982)Goldschlager, Shaw, and
  Staples}]{goldschlager1982}
Goldschlager LM, Shaw RA, Staples J (1982) The maximum flow problem is log
  space complete for {P}. Theoretical Computer Science 21(1):105--111

\bibitem[{Greenlaw et~al(1995)Greenlaw, Hoover, and Ruzzo}]{greenlaw1995limits}
Greenlaw R, Hoover HJ, Ruzzo WL (1995) Limits to parallel computation:
  {P}-completeness theory. Oxford University Press

\bibitem[{Haase et~al(2023)Haase, Hertrich, and Loho}]{haase2023lower}
Haase CA, Hertrich C, Loho G (2023) Lower bounds on the depth of integral
  {ReLU} neural networks via lattice polytopes. In: International Conference on
  Learning Representations (ICLR)

\bibitem[{Hanin(2019)}]{hanin2019universal}
Hanin B (2019) Universal function approximation by deep neural nets with
  bounded width and {ReLU} activations. Mathematics 7(10):992

\bibitem[{Hanin and Rolnick(2019)}]{hanin2019complexity}
Hanin B, Rolnick D (2019) Complexity of linear regions in deep networks. In:
  International Conference on Machine Learning

\bibitem[{Hanin and Sellke(2017)}]{hanin2017approximating}
Hanin B, Sellke M (2017) Approximating continuous functions by {ReLU} nets of
  minimal width. arXiv:171011278

\bibitem[{Hartmanis and Simon(1974)}]{hartmanis1974power}
Hartmanis J, Simon J (1974) On the power of multiplication in random access
  machines. In: 15th Annual Symposium on Switching and Automata Theory ({SWAT}
  1974), IEEE, pp 13--23

\bibitem[{Hertrich and Sering(2023)}]{hertrich2023relu}
Hertrich C, Sering L (2023) {ReLU} neural networks of polynomial size for exact
  maximum flow computation. In: International Conference on Integer Programming
  and Combinatorial Optimization, Springer, pp 187--202

\bibitem[{Hertrich and Skutella(2021)}]{knapsackPaper}
Hertrich C, Skutella M (2021) Provably good solutions to the knapsack problem
  via neural networks of bounded size. AAAI Conference on Artificial
  Intelligence

\bibitem[{Hertrich et~al(2021)Hertrich, Basu, Di~Summa, and
  Skutella}]{hertrich2021towards}
Hertrich C, Basu A, Di~Summa M, et~al (2021) Towards lower bounds on the depth
  of {ReLU} neural networks. Advances in Neural Information Processing Systems
  34:3336--3348

\bibitem[{Hopfield and Tank(1985)}]{hopfield1985neural}
Hopfield JJ, Tank DW (1985) ``{N}eural'' computation of decisions in
  optimization problems. Biological Cybernetics 52(3):141--152

\bibitem[{Hornik(1991)}]{hornik1991approximation}
Hornik K (1991) Approximation capabilities of multilayer feedforward networks.
  Neural networks 4(2):251--257

\bibitem[{Jerrum and Snir(1982)}]{jerrum1982some}
Jerrum M, Snir M (1982) Some exact complexity results for straight-line
  computations over semirings. Journal of the ACM (JACM) 29(3):874--897

\bibitem[{Jukna(2015)}]{jukna2015lower}
Jukna S (2015) Lower bounds for tropical circuits and dynamic programs. Theory
  of Computing Systems 57(1):160--194

\bibitem[{Jukna and Seiwert(2019)}]{jukna2019greedy}
Jukna S, Seiwert H (2019) Greedy can beat pure dynamic programming. Information
  Processing Letters 142:90--95

\bibitem[{Kennedy and Chua(1988)}]{kennedy1988neural}
Kennedy MP, Chua LO (1988) Neural networks for nonlinear programming. IEEE
  Transactions on Circuits and Systems 35(5):554--562

\bibitem[{Khalife and Basu(2022)}]{khalife2022neural}
Khalife S, Basu A (2022) Neural networks with linear threshold activations:
  structure and algorithms. In: International Conference on Integer Programming
  and Combinatorial Optimization, Springer, pp 347--360

\bibitem[{Khalil et~al(2017)Khalil, Dai, Zhang, Dilkina, and
  Song}]{Khalil:LearningCOoverGraphs}
Khalil E, Dai H, Zhang Y, et~al (2017) Learning combinatorial optimization
  algorithms over graphs. Advances in neural information processing systems 30

\bibitem[{Kool et~al(2019)Kool, van Hoof, and Welling}]{kool2019attention}
Kool W, van Hoof H, Welling M (2019) Attention, learn to solve routing
  problems! In: International Conference on Learning Representations

\bibitem[{Korte and Vygen(2008)}]{kortevygen}
Korte B, Vygen J (2008) Combinatorial Optimization: Theory and Algorithms, 4th
  edn. Springer

\bibitem[{{LeCun} et~al(2015){LeCun}, Bengio, and
  Hinton}]{LeCunBengioHinton:DeepLearning}
{LeCun} Y, Bengio Y, Hinton G (2015) Deep learning. Nature 521:436--444

\bibitem[{Liang and Srikant(2017)}]{liang2017deep}
Liang S, Srikant R (2017) Why deep neural networks for function approximation?
  In: International Conference on Learning Representations

\bibitem[{Lodi and Zarpellon(2017)}]{lodi2017learning}
Lodi A, Zarpellon G (2017) On learning and branching: a survey. {TOP}
  25(2):207--236

\bibitem[{McCormick(1999)}]{mccormick1999fast}
McCormick ST (1999) Fast algorithms for parametric scheduling come from
  extensions to parametric maximum flow. Operations Research 47(5):744--756

\bibitem[{Montufar et~al(2014)Montufar, Pascanu, Cho, and
  Bengio}]{montufar2014regions}
Montufar GF, Pascanu R, Cho K, et~al (2014) On the number of linear regions of
  deep neural networks. Advances in neural information processing systems 27

\bibitem[{Mukherjee and Basu(2017)}]{mukherjee2017lower}
Mukherjee A, Basu A (2017) Lower bounds over boolean inputs for deep neural
  networks with {ReLU} gates. arXiv:171103073

\bibitem[{Nazemi and Omidi(2012)}]{nazemi2012}
Nazemi A, Omidi F (2012) A capable neural network model for solving the maximum
  flow problem. Journal of Computational and Applied Mathematics
  236(14):3498--3513

\bibitem[{Nguyen et~al(2018)Nguyen, Mukkamala, and Hein}]{nguyen2018neural}
Nguyen Q, Mukkamala MC, Hein M (2018) Neural networks should be wide enough to
  learn disconnected decision regions. In: International Conference on Machine
  Learning

\bibitem[{{Nowak} et~al(2017){Nowak}, {Villar}, {Bandeira}, and
  {Bruna}}]{nowak:quadrAssignment}
{Nowak} A, {Villar} S, {Bandeira} AS, et~al (2017) {Revised Note on Learning
  Algorithms for Quadratic Assignment with Graph Neural Networks}.
  arXiv:170607450

\bibitem[{Orlin(2013)}]{orlin2013}
Orlin JB (2013) Max flows in {O}(nm) time, or better. In: Proceedings of the
  Forty-Fifth Annual ACM Symposium on Theory of Computing (STOC '13).
  Association for Computing Machinery, pp 765--774

\bibitem[{Parberry et~al(1994)Parberry, Garey, and Meyer}]{parberry1994circuit}
Parberry I, Garey MR, Meyer A (1994) Circuit complexity and neural networks.
  MIT Press

\bibitem[{Pascanu et~al(2014)Pascanu, Montufar, and Bengio}]{pascanu2014number}
Pascanu R, Montufar G, Bengio Y (2014) On the number of inference regions of
  deep feed forward networks with piece-wise linear activations. In:
  International Conference on Learning Representations

\bibitem[{Pratt et~al(1974)Pratt, Rabin, and
  Stockmeyer}]{pratt1974characterization}
Pratt VR, Rabin MO, Stockmeyer LJ (1974) A characterization of the power of
  vector machines. In: Proceedings of the sixth annual {ACM} Symposium on
  Theory of Computing ({STOC}), pp 122--134

\bibitem[{Raghu et~al(2017)Raghu, Poole, Kleinberg, Ganguli, and
  Dickstein}]{raghu2017expressive}
Raghu M, Poole B, Kleinberg J, et~al (2017) On the expressive power of deep
  neural networks. In: International Conference on Machine Learning

\bibitem[{Rothvo{\ss}(2017)}]{rothvoss2017matching}
Rothvo{\ss} T (2017) The matching polytope has exponential extension
  complexity. Journal of the ACM (JACM) 64(6):1--19

\bibitem[{Safran and Shamir(2017)}]{safran2017depth}
Safran I, Shamir O (2017) Depth-width tradeoffs in approximating natural
  functions with neural networks. In: International Conference on Machine
  Learning

\bibitem[{Sch{\"o}nhage(1979)}]{schonhage1979power}
Sch{\"o}nhage A (1979) On the power of random access machines. In:
  International Colloquium on Automata, Languages, and Programming, Springer,
  pp 520--529

\bibitem[{Serra et~al(2018)Serra, Tjandraatmadja, and
  Ramalingam}]{serra2018bounding}
Serra T, Tjandraatmadja C, Ramalingam S (2018) Bounding and counting linear
  regions of deep neural networks. In: International Conference on Machine
  Learning

\bibitem[{Shalev-Shwartz and Ben-David(2014)}]{Shalev2014:UnderstandingML}
Shalev-Shwartz S, Ben-David S (2014) Understanding machine learning: From
  theory to algorithms. Cambridge University Press

\bibitem[{Shamos(1979)}]{shamos1979}
Shamos MI (1979) Computational geometry. PhD thesis, Yale University

\bibitem[{Shawe-Taylor et~al(1992)Shawe-Taylor, Anthony, and
  Kern}]{shawe1992classes}
Shawe-Taylor JS, Anthony MH, Kern W (1992) Classes of feedforward neural
  networks and their circuit complexity. Neural networks 5(6):971--977

\bibitem[{Shpilka and Yehudayoff(2010)}]{shpilka2010arithmetic}
Shpilka A, Yehudayoff A (2010) Arithmetic circuits: A survey of recent results
  and open questions. Now Publishers Inc

\bibitem[{Smith(1999)}]{Smith:NNforCOreview}
Smith KA (1999) Neural networks for combinatorial optimization: A review of
  more than a decade of research. INFORMS Journal on Computing 11(1):15--34

\bibitem[{Telgarsky(2015)}]{Telgarsky15}
Telgarsky M (2015) Representation benefits of deep feedforward networks.
  arXiv:150908101

\bibitem[{Telgarsky(2016)}]{telgarsky2016benefits}
Telgarsky M (2016) Benefits of depth in neural networks. In: Conference on
  Learning Theory, pp 1517--1539

\bibitem[{Vinyals et~al(2015)Vinyals, Fortunato, and
  Jaitly}]{Vinyals:PointerNetworks}
Vinyals O, Fortunato M, Jaitly N (2015) Pointer networks. Advances in neural
  information processing systems 28

\bibitem[{Williamson(2019)}]{williamson_2019}
Williamson DP (2019) Network Flow Algorithms. Cambridge University Press

\bibitem[{Yarotsky(2017)}]{yarotsky2017error}
Yarotsky D (2017) Error bounds for approximations with deep relu networks.
  Neural Networks 94:103--114

\bibitem[{Zhang et~al(2021)Zhang, Bengio, Hardt, Recht, and
  Vinyals}]{zhang2021understanding}
Zhang C, Bengio S, Hardt M, et~al (2021) Understanding deep learning (still)
  requires rethinking generalization. Communications of the ACM 64(3):107--115

\end{thebibliography}

\end{document}